\documentclass[11pt,a4paper]{article}
\usepackage[T1]{fontenc}
\usepackage[utf8]{inputenc}
\IfFileExists{lmodern.sty}{\usepackage{lmodern}}{}
\usepackage[english]{babel}
\usepackage[margin=2.4cm]{geometry}
\usepackage{amsmath,amssymb,amsthm}
\usepackage{booktabs,longtable,array,float,verbatim,graphicx,color}
\usepackage[numbers,sort&compress]{natbib}
\usepackage{tikz}
\usetikzlibrary{positioning,arrows.meta,shapes.geometric,fit,calc}
\usepackage[colorlinks,urlcolor=blue,linkcolor=blue,citecolor=blue]{hyperref}
\newtheorem{proposition}{Proposition}
\newtheorem{corollary}{Corollary}
\newtheorem{assumption}{Assumption}
\theoremstyle{definition}
\newtheorem{definition}{Definition}
\newcommand{\tr}{\operatorname{tr}}

\title{Marginal utility as a resource-allocation rule for sovereign document AI:\\ KV-cache compression, model merging, frugal inference,\\ and a three-control efficiency point}
\author{Caroline Gans Combe\\ \small INSEEC Business School, Omnes Education, Paris, France}
\date{Working paper, version 2 (work in progress), 6 November 2026}

\begin{document}
\maketitle

\begin{abstract}
\noindent Deploying language and vision-language models under memory, energy and data-sovereignty constraints requires a chain of allocation decisions that current practice answers with separate heuristics. Part~I shows that these heuristics derive from a single principle, diminishing marginal utility: the eigenvalues of the attention-weighted second-moment operator of the value projections form an exact marginal-utility schedule for low-rank KV-cache compression, whose tail equals, for fixed attention weights, the expected attention-weighted reconstruction error; a shared budget equalizes weighted marginal eigenvalues across heads; and rank selection in regularized matrix factorization, trimming in model merging and stopping in multi-pass inference follow the same threshold rule, exactly or under stated assumptions. In an industrial geoscience pipeline restricted to local open-weight models, a frugal classifier does not differ significantly from its proprietary teacher against a two-pass expert audit while running several hundred times faster, and uniform-density merging is shown to collapse to input-independent outputs; layer-dependent densities removed the collapse on a diagnostic sample, but a later full run of the merge recipe failed the same test, and the conditions of the result are discussed. Part~II gathers the model-selection analysis under a control and a provenance criterion, the benchmark protocol, the interface with a partner parsing pipeline, axiological implications and open problems. Part~III, a work in progress, restates the research objective after the extension of the project to engineering drawings of nuclear installations: a sovereign multi-purpose collection and vision model whose input-token, output-token and energy use is bounded by marginal utility, with an efficiency point at which one more unit of any of the three resources adds no effectiveness; it reports the screening of candidate components under a European or Canadian provenance rule and the status of the work at 6 November 2026.
\end{abstract}

\noindent\textbf{Keywords.} marginal utility; KV cache; low-rank approximation; model merging; frugal inference; energy efficiency; document analysis; sovereign AI; axiological displacement.

\paragraph{Version note.} This version supersedes version~11 of 14 September 2026 (HAL hal-05749707v1). It is aligned with the companion journal manuscript and corrects the following points of version~11: the dimensions of the operator $WVW^{\top}$, now defined with $V=X^{\top}X/n$ the second-moment matrix of token representations; Proposition~1, whose statement equated a squared Frobenius error with a sum of eigenvalues and is now stated through the Ky Fan maximum principle; the connection to the singular value decomposition, whose equation~(14) was incorrect and is replaced by the identity $\lambda_j(M)=\sigma_j^{2}(XW^{\top})/n$; the expected output covariance of equation~(19), which omitted cross terms and is replaced by the exact attention-weighted reconstruction identity of Proposition~\ref{prop:kvbound}; the number of consolidated level-2 classes, which is 26 and not 22; the description of the expert audit; the projected targets for the merged architecture, now withdrawn; and the status of the merge recipe, whose later full run failed the non-degeneracy test.

\tableofcontents

\part{Allocation rule and measured evidence}\label{part:main}
\section{Introduction}\label{sec:intro}

Large language models (LLMs) and vision-language models are increasingly deployed in settings where computation is not free and data cannot leave the operator's infrastructure. Geological surveys held by operators in the nuclear fuel cycle are a representative case: the documents are heterogeneous scans of maps, borehole logs, and geophysical surveys; their transmission to non-European servers can fall under dual-use export controls \cite{eudualuse2021}; and the resulting models must run on a single commodity accelerator whose energy consumption is reportable. Under such constraints every stage of inference becomes an allocation problem. The key-value (KV) cache of an autoregressive transformer \cite{vaswani2017attention} grows linearly with context length and must be compressed or pruned \cite{pope2023efficient,shi2024kvreview}; fine-tuned specialists must be combined into one deployable model without retraining \cite{yadav2023ties,yang2024mergesurvey}; and an adaptive pipeline must decide whether an input deserves a second, more expensive pass \cite{chen2024frugalgpt,snell2025scaling}.

Each of these decisions currently has its own heuristic. Eviction policies retain tokens with high cumulative attention \cite{zhang2023h2o,liu2023scissorhands,li2024snapkv}; low-rank methods project keys and values onto principal subspaces and search ranks per layer \cite{saxena2024eigen,chang2025palu}; budget-allocation methods redistribute memory across heads \cite{feng2025adakv}; merging methods trim small task-vector entries \cite{yadav2023ties,yu2024dare}; cascades stop when a scorer is confident \cite{chen2024frugalgpt,schuster2022calm}. The heuristics work, but they are justified locally and cannot easily be compared, audited, or tuned jointly by an operator who must trade quality against memory, energy, and legal exposure.

This paper argues that these heuristics share a common structure that microeconomics identified long ago: an agent allocating a scarce resource across goods of diminishing marginal utility should keep acquiring a good only while its marginal utility exceeds its shadow price, and at the optimum marginal utilities per unit cost are equal across goods \cite{gossen1854,mascolell1995}. The attention literature has used economic language informally, and rational-inattention models treat attention itself as a scarce resource \cite{sims2003rational}, but the correspondence has not been made precise for transformer inference. We make it precise through a single object, the projected second-moment operator $M=WVW^{\top}$, whose spectrum is directly computable from a trained model's weights and activations.

The paper makes three contributions. First, on the theoretical side, it proves that the eigenvalues of the attention-weighted value operator constitute an exact marginal-utility schedule for low-rank KV-cache compression, with an identity for the attention-weighted reconstruction error and a bound on the output error (Propositions~\ref{prop:schedule} and~\ref{prop:kvbound}); it derives a threshold allocation rule (Proposition~\ref{prop:unified}) whose corollary is that, under a shared memory budget, optimal per-head ranks equalize marginal eigenvalues across heads rather than applying a uniform compression ratio (Corollary~\ref{cor:equimarginal}); and it shows that rank selection in regularized matrix factorization, trimming in TIES merging, and stopping in multi-pass inference are instances of the same rule, stating explicitly which instances are exact and which rest on modeling assumptions. Second, on the methodological side, it translates the rule into a layer-dependent trimming schedule for merging a vision-language model with a language model, and into a quality-floor stopping rule for a multi-pass extraction pipeline. Third, on the empirical side, it reports three measured experiments in an industrial geoscience pipeline: a reference measurement of proprietary multimodal services on map georeferencing, whose prompt carried the district name, a limitation reported with the measurement, a paired evaluation, against the same expert audit, of a frugal hierarchical classifier and the proprietary model whose labels it was trained on, including a confidence cascade between the two, and a diagnostic showing that uniform-density merging produces input-independent outputs, which layer-dependent densities removed on that diagnostic sample but which a later full run of the merge recipe reproduced.

The scope of the evidence is deliberately stated. The end-to-end extraction benchmark of the merged model with low-rank adaptation (LoRA) fine-tuning \cite{hu2022lora} is not reported here, and no projected figure for that configuration is given; the full model-selection analysis and the interface with a partner document-parsing pipeline are given in Part~\ref{part:supp} and play no role in the conclusions of Part~\ref{part:main}.

\emph{Final objective.} The contributions above serve a broader objective, stated here to frame the work rather than as a result: a sovereign, multi-purpose collection and vision model for complex technical documents, whose resource use is bounded by marginal utility along three controls, input tokens, output tokens and energy. Section~\ref{sec:outlook} formalizes this objective as a generalization of the threshold rule.

Section~\ref{sec:related} reviews related work. Section~\ref{sec:method} develops the framework. Section~\ref{sec:design} describes the application and experimental design, Section~\ref{sec:results} reports the results, Section~\ref{sec:discussion} discusses implications and limitations, Section~\ref{sec:outlook} states the outlook, and Section~\ref{sec:conclusion} concludes.

\section{Background and Related Work}\label{sec:related}

\subsection{Marginal Utility and Spectral Rank Selection}
Diminishing marginal utility states that the increment of utility from one more unit of a good decreases with the quantity already held; under a budget, the optimal bundle equalizes marginal utility per unit price across goods \cite{mascolell1995}. In discrete choice the first-order condition becomes a threshold rule: acquire unit $j$ if and only if its marginal utility exceeds its shadow price. Spectral methods possess the same structure. By the Eckart--Young theorem \cite{eckart1936} the best rank-$k$ approximation of a matrix retains its $k$ leading singular components, and the error decrease contributed by each additional component is non-increasing. Rank selection has been studied as a thresholding problem, from the elbow heuristic to the optimal hard threshold of Gavish and Donoho \cite{gavish2014optimal} and the singular-value soft-thresholding induced by nuclear-norm regularization in matrix completion \cite{mazumder2010spectral}. Recommender systems have long interpreted latent factors as preference dimensions \cite{koren2009matrix}. None of these works connects the threshold to a resource price in transformer inference.

\subsection{KV-Cache Management}
The KV cache stores the keys and values of past tokens for every layer and head; its size scales with depth, head count, head dimension, and context length, and it dominates memory at long context \cite{pope2023efficient,Kwon2023}. Architectural remedies share key and value heads \cite{shazeer2019mqa,ainslie2023gqa} or cache a low-dimensional latent \cite{deepseek2024v2}. Post-training remedies fall into three families \cite{shi2024kvreview}. Token eviction keeps heavy hitters by cumulative attention \cite{zhang2023h2o}, exploits the persistence of importance \cite{liu2023scissorhands}, retains attention sinks and a recent window \cite{Xiao2024}, profiles heads to choose a policy \cite{ge2024fastgen}, selects tokens from an observation window \cite{li2024snapkv}, or allocates budget by layer \cite{cai2024pyramidkv}, by head \cite{feng2025adakv}, or by head type \cite{xiao2025duoattention}. Quantization reduces bit width \cite{liu2024kivi}. Low-rank methods compress the hidden dimension: Eigen Attention performs attention in a principal subspace estimated from calibration data \cite{saxena2024eigen}, Palu factorizes the key and value projections with a rank search guided by Fisher information \cite{chang2025palu}, and Loki exploits the low rank of keys to select tokens \cite{singhania2024loki}. Head-wise adaptive allocation is empirically superior to uniform allocation \cite{feng2025adakv,saxena2024eigen}, which the present framework explains as a consequence of equalizing marginal utilities (Corollary~\ref{cor:equimarginal}). The framework also differs in one operational respect: it weights the calibration operator by the attention distribution, so that directions used by high-attention tokens receive more weight than in an unweighted principal component analysis.

\subsection{Low-Rank Adaptation and Model Merging}
Fine-tuning updates have low intrinsic dimension \cite{aghajanyan2021intrinsic}, which LoRA exploits with rank-$r$ additive corrections \cite{hu2022lora}; AdaLoRA reallocates rank across matrices according to importance scores under a total budget \cite{zhang2023adalora}, and truncation-aware singular value decomposition compresses pretrained weights \cite{wang2025svdllm}. Merging combines fine-tuned models in weight space, by averaging \cite{wortsman2022soups}, Fisher-weighted averaging \cite{matena2022fisher}, task arithmetic \cite{ilharco2023task}, trimming and sign election \cite{yadav2023ties}, random drop and rescale \cite{yu2024dare}, or learned layer-wise coefficients \cite{yang2024adamerging}; composing multimodal models by merging their language backbones is a recent extension \cite{chen2024composition}. Layer-wise merging coefficients have been learned from unlabeled data \cite{yang2024adamerging}; the present work instead motivates layer-dependent trimming densities from the allocation rule and documents a failure mode of uniform densities in a vision-language setting.

\subsection{Adaptive Inference and Frugal AI}
Adaptive computation allocates effort per input, through halting \cite{graves2016act}, early exit \cite{schuster2022calm}, cascades of increasingly expensive models \cite{chen2024frugalgpt}, or test-time compute scaling \cite{snell2025scaling}. Optimal stopping theory \cite{chow1971great} provides the formal backdrop. The energy and carbon cost of training and inference is now measured and reported \cite{Strubell2019,Schwartz2020,henderson2020energy,patterson2021carbon,samsi2023watts,luccioni2024power}, and value-sensitive critiques have questioned the default toward scale \cite{Bender2021,Bommasani2021}. Document AI models jointly encode layout, text, and image \cite{huang2022layoutlmv3,kim2022donut}, map processing has a long history \cite{chiang2014mapsurvey}, and domain LLMs exist for geoscience text \cite{deng2024k2}, but frugal, locally deployable pipelines for georeferencing geological documents remain little studied.

\subsection{Positioning}
The contribution is not a new compression or merging algorithm that outperforms the methods above on standard benchmarks. It is a unifying allocation principle with exact statements where they hold, an explanation of why adaptive allocation dominates uniform allocation, and evidence from a constrained industrial deployment that the principle yields usable decisions. An earlier talk introduced the conceptual link between marginal utility, matrix factorization, and the KV cache \cite{gans2025cio}; the propositions, the merging and stopping instances, and all experiments reported here are new.

\section{The Marginal-Utility Allocation Framework}\label{sec:method}

\subsection{Projected Second-Moment Operator}
Let $X\in\mathbb{R}^{n\times d}$ collect $n$ token representations as rows, let $V=\frac{1}{n}X^{\top}X\in\mathbb{R}^{d\times d}$ be their second-moment matrix, and let $W\in\mathbb{R}^{k\times d}$ be a learned projection. The projected second-moment operator is
\begin{equation}
M = W V W^{\top}\in\mathbb{R}^{k\times k},\qquad M=Q\Lambda Q^{\top},
\label{eq:wvw}
\end{equation}
with eigenvalues $\lambda_1\geq\lambda_2\geq\cdots\geq\lambda_k\geq 0$ and orthonormal eigenvectors $q_1,\dots,q_k$. Throughout, $\lambda_j$ denotes an eigenvalue and $\gamma$ a shadow price.

\begin{proposition}[Spectral marginal-utility schedule]\label{prop:schedule}
For every $j\leq k$, the maximal second moment captured by a $j$-dimensional orthogonal projection is $\max_{P^{\top}P=I_j}\tr(P^{\top}MP)=\sum_{i\leq j}\lambda_i$, attained by $P=[q_1,\dots,q_j]$. The marginal utility of the $j$-th direction is therefore $\lambda_j$, a non-increasing sequence. Moreover $\lambda_j(M)=\sigma_j^2(XW^{\top})/n\leq\|W\|_2^2\,\lambda_j(V)$.
\end{proposition}
\begin{proof}
The first statement is the Ky Fan maximum principle \cite{horn2012matrix}. The identity follows from $M=\frac{1}{n}(XW^{\top})^{\top}(XW^{\top})$, and the bound from $\sigma_j(XW^{\top})\leq\sigma_j(X)\|W\|_2$.
\end{proof}
Proposition~\ref{prop:schedule} is a restatement of classical results; its value is interpretive. The bound shows that a projection cannot create second moment absent from the data, so compression can only remove utility.

\subsection{Attention-Weighted Value Operator and Cache Compression}
Consider one attention head with value projection $W_V\in\mathbb{R}^{d\times d_h}$, cached values $v_i=x_iW_V$, and attention weights $a_{t,i}\geq 0$, $\sum_i a_{t,i}=1$, for query $t$. The head output is $o_t=\sum_i a_{t,i}v_i$. Define the attention-weighted input moment $\hat V_t=\sum_i a_{t,i}x_i^{\top}x_i$ and
\begin{equation}
M_t=W_V^{\top}\hat V_tW_V=\textstyle\sum_i a_{t,i}\,v_i^{\top}v_i ,
\label{eq:Mt}
\end{equation}
which has the form~\eqref{eq:wvw} with $W=W_V^{\top}$. For a calibration set of queries let $\bar M=\mathbb{E}_t[M_t]$ with eigenvectors $q_j$, and let $P_r=[q_1,\dots,q_r]$. Low-rank caching stores $v_iP_r\in\mathbb{R}^{r}$ instead of $v_i\in\mathbb{R}^{d_h}$ and reconstructs $\tilde v_i=v_iP_rP_r^{\top}$ and $\tilde o_t=\sum_i a_{t,i}\tilde v_i$.

\begin{proposition}[Compression error as foregone utility]\label{prop:kvbound}
With the notation above,
\begin{equation}
\mathbb{E}_t\Big[\textstyle\sum_i a_{t,i}\|v_i-\tilde v_i\|^2\Big]=\sum_{j>r}\lambda_j(\bar M),
\label{eq:identity}
\end{equation}
and $\mathbb{E}_t\|o_t-\tilde o_t\|^2\leq\sum_{j>r}\lambda_j(\bar M)$. No rank-$r$ projection achieves a smaller left-hand side in~\eqref{eq:identity}.
\end{proposition}
\begin{proof}
Let $R=I-P_rP_r^{\top}$. Then $\sum_i a_{t,i}\|v_iR\|^2=\tr(R\,M_t\,R)=\tr(RM_t)$, and taking expectations gives $\tr(R\bar M)=\tr(\bar M)-\tr(P_r^{\top}\bar MP_r)=\sum_{j>r}\lambda_j(\bar M)$; optimality follows from Proposition~\ref{prop:schedule}. Since $o_t-\tilde o_t=\sum_i a_{t,i}(v_i-\tilde v_i)$ is a convex combination, Jensen's inequality bounds its squared norm by $\sum_i a_{t,i}\|v_i-\tilde v_i\|^2$.
\end{proof}
The expected loss from compressing at rank $r$ is thus exactly the tail of the spectrum of $\bar M$: each discarded eigenvalue is a unit of foregone utility. Unweighted principal component analysis of the values \cite{saxena2024eigen} corresponds to replacing $a_{t,i}$ by uniform weights; Proposition~\ref{prop:kvbound} indicates that the attention-weighted operator is the one whose tail controls the output error. Its scope is deliberately narrow: the identity and the bound hold for fixed attention weights and for compression of the values only. They control the error of one head's output under the observed attention, not the end-to-end quality of generation, which also depends on the keys, on attention recomputed after compression and on the propagation of errors across layers.

\subsection{Threshold Rule and Equimarginal Allocation}
\begin{proposition}[Threshold allocation rule]\label{prop:unified}
Let $\lambda_1\geq\cdots\geq\lambda_n\geq 0$ be the marginal utilities of $n$ sources and let each retained source cost $c>0$.
(i) \emph{Price form.} For a price $\gamma\geq0$ per unit of cost, $\max_{S}\sum_{j\in S}(\lambda_j-\gamma c)$ is solved by $S^{*}=\{1,\dots,k^{*}\}$ with
\begin{equation}
k^{*}=\max\{k:\lambda_k\geq\gamma c\},
\label{eq:kstar}
\end{equation}
and $k^{*}=0$ if $\lambda_1<\gamma c$; sources with $\lambda_j=\gamma c$ contribute zero and may equally be dropped.
(ii) \emph{Budget form.} For a budget $B$, $\max_{S}\sum_{j\in S}\lambda_j$ subject to $c|S|\leq B$ is solved by the prefix of length
\begin{equation}
m=\min\{n,\lfloor B/c\rfloor\},
\label{eq:kbudget}
\end{equation}
with $m=0$ if $B<c$. When $\lambda_m>\lambda_{m+1}$ (with $\lambda_{n+1}=0$), the budget solution coincides with the price solution~\eqref{eq:kstar} for every $\gamma c\in(\lambda_{m+1},\lambda_m]$; when $\lambda_m=\lambda_{m+1}$, the threshold reproduces it only with an explicit tie-breaking rule at the cutoff. The threshold $\gamma$ is in general not unique, and when $B/c$ is not an integer the Lagrangian dual value exceeds the optimum, so that $\gamma$ reproduces the optimal selection without being an exact discrete shadow price.
\end{proposition}
The proof is given in the Appendix. With unit cost normalized to one, the price form reads $k^{*}=\max\{k:\lambda_k\geq\gamma\}$; this is the form in which the rule recurs below, the price $\gamma$ expressing the operator's valuation of the resource. When costs differ across sources the problem is a knapsack whose linear relaxation is solved greedily by the ratio $\lambda_j/c_j$; the threshold then applies to utility per unit cost, which is the textbook equal-marginal-utility-per-price condition \cite{mascolell1995}.

\begin{corollary}[Equimarginal allocation across heads]\label{cor:equimarginal}
Let head $h$ of layer $\ell$ have schedule $\lambda^{(\ell,h)}_1\geq\lambda^{(\ell,h)}_2\geq\cdots$, let $\omega_\ell=\|W_O^{(\ell)}\|_2^2$ where $W_O^{(\ell)}$ is the output projection of layer $\ell$, and let a shared budget allow $B/c$ retained dimensions. Minimizing the bound $\sum_{\ell,h}\omega_\ell\sum_{j>r_{\ell h}}\lambda^{(\ell,h)}_j$ on the layer-output reconstruction error retains the $m=\min\{\sum_{\ell,h}d_h,\lfloor B/c\rfloor\}$ largest weighted eigenvalues across all heads; up to ties at the cutoff, this is $r^{*}_{\ell h}=\max\{j:\omega_\ell\lambda^{(\ell,h)}_j\geq\gamma c\}$ with a single $\gamma$ for all heads. A uniform compression ratio need not be optimal; it is optimal exactly when this pooled selection retains the same number of directions in every head, which can happen for distinct weighted spectra.
\end{corollary}
\begin{proof}
The layer output error equals $[e_1,\dots,e_H]W_O^{(\ell)}$, whose squared norm is at most $\omega_\ell\sum_h\|e_h\|^2$; each $\mathbb{E}\|e_h\|^2$ is bounded by Proposition~\ref{prop:kvbound}. Maximizing retained weighted eigenmass under a common cardinality budget is solved by selecting the $m$ largest pooled values $\omega_\ell\lambda^{(\ell,h)}_j$ (Proposition~\ref{prop:unified}(ii)); because each weighted schedule is non-increasing, a selection with ties broken consistently within each head retains a prefix of every head, which gives the common threshold.
\end{proof}

Fig.~\ref{fig:equimarginal} illustrates the corollary. A head with a concentrated spectrum, such as the positional or syntactic heads identified by head analysis \cite{voita2019analyzing}, receives a low rank; a head with a diffuse spectrum receives a high rank. The corollary supplies a theoretical rationale for head-wise and layer-wise adaptive budgets \cite{feng2025adakv,saxena2024eigen,chang2025palu}; the weight $\omega_\ell$ is a local bound and ignores error propagation through subsequent layers, which is a limitation shared with those methods.

\begin{figure}[t]
\centering
\begin{tikzpicture}[x=0.36cm,y=0.42cm,font=\scriptsize]
\foreach \i/\h in {1/6.0,2/2.2,3/1.1,4/0.6,5/0.35,6/0.2,7/0.12,8/0.08}{
  \pgfmathsetmacro{\fillc}{\h>=0.9?"black!55":"black!12"}
  \fill[\fillc] (\i-0.4,0) rectangle (\i+0.4,\h);
  \draw (\i-0.4,0) rectangle (\i+0.4,\h);}
\foreach \i/\h in {1/2.6,2/2.3,3/2.0,4/1.7,5/1.4,6/1.15,7/0.95,8/0.7}{
  \pgfmathsetmacro{\fillc}{\h>=0.9?"black!55":"black!12"}
  \fill[\fillc] (\i+10-0.4,0) rectangle (\i+10+0.4,\h);
  \draw (\i+10-0.4,0) rectangle (\i+10+0.4,\h);}
\draw[dashed,thick] (0.2,0.9) -- (19,0.9) node[right] {$\gamma c$};
\draw[->] (0.2,0) -- (19,0);
\node at (4.5,-0.8) {Head A: concentrated, $r^{*}=3$};
\node at (14.5,-0.8) {Head B: diffuse, $r^{*}=7$};
\node[rotate=90] at (-0.8,3) {$\omega_\ell\lambda_j$};
\end{tikzpicture}
\caption{Equimarginal allocation (Corollary~\ref{cor:equimarginal}). A single threshold $\gamma c$ applied to the weighted eigenvalue schedules of two heads retains the dark bars. The head with the concentrated spectrum receives fewer dimensions than the head with the diffuse spectrum, whereas a uniform compression ratio would retain the same number in both.}
\label{fig:equimarginal}
\end{figure}

\subsection{Other Instances of the Rule}\label{sec:instances}
Table~\ref{tab:unified} lists the instances used in this paper and states the status of each correspondence.

\emph{Matrix factorization.} For a fully observed $R$, minimizing $\|R-PQ^{\top}\|_F^2+\lambda(\|P\|_F^2+\|Q\|_F^2)$ is equivalent to minimizing $\|R-Z\|_F^2+2\lambda\|Z\|_{*}$, whose solution soft-thresholds the singular values of $R$ at $\lambda$ \cite{mazumder2010spectral}. The retained rank is $\max\{j:\sigma_j(R)>\lambda\}$: the regularization coefficient acts as the price of the price form~\eqref{eq:kstar}, a penalty threshold on singular values, and not necessarily as the shadow price of a discrete rank budget.

\emph{Token eviction.} If the marginal utility of token $i$ is defined as the quality loss from deleting it, cumulative attention is a proxy for it, and heavy-hitter eviction \cite{zhang2023h2o} applies the threshold rule to that proxy. The correspondence is heuristic for two reasons. Attention scores are not additive utilities; and attention measures how much a token is used, not how much it adds given the others: in redundant text or images, two tokens carrying the same information may both receive high attention while the second has near-zero marginal utility, and conversely. Proposition~\ref{prop:kvbound} avoids this confusion by operating on directions of the value space rather than on tokens: redundant tokens load on the same eigendirections, whose eigenvalue counts their shared contribution once.

\emph{Trimming in model merging.} TIES keeps the fraction $\rho$ of task-vector entries $\delta_p=\theta_p-\theta_{\mathrm{base},p}$ of largest magnitude \cite{yadav2023ties}. Under a diagonal quadratic approximation of the loss with Fisher information $F$ \cite{matena2022fisher}, dropping entry $p$ costs approximately $\tfrac12F_{pp}\delta_p^2$; magnitude trimming is the threshold rule when $F$ is isotropic, and a density $\rho_\ell$ specific to layer $\ell$ corresponds to a layer-specific shadow price. Uniform density across layers therefore need not be optimal when the layers' marginal-utility schedules differ, which Section~\ref{sec:res_ties} examines.

\emph{Multi-pass inference.} A pipeline may run up to $K$ passes of increasing cost on an input. Let $\hat\mu_k\in[0,1]$ be a quality score after pass $k$ and $c$ the cost per pass.
\begin{assumption}\label{ass:passes}
The expected score gain of pass $k$ equals $g_k\,\mathbb{E}[\hat\mu_1]$ for a non-increasing sequence $g_k$.
\end{assumption}
Under Assumption~\ref{ass:passes} the stopping problem is monotone, the one-step look-ahead rule is optimal \cite{chow1971great}, and the optimal number of passes is $\tau^{*}=\max\{k:g_k\geq\gamma c/\mathbb{E}[\hat\mu_1]\}$, an instance of~\eqref{eq:kstar}. Assumption~\ref{ass:passes} is a modeling choice, not a consequence of Proposition~\ref{prop:schedule}; the deployed pipeline (Section~\ref{sec:pipeline}) uses the observable counterpart, a quality floor below which another pass is triggered.

\begin{table}[t]
\caption{Instances of the Threshold Rule and Status of Each Correspondence}
\label{tab:unified}
\centering
\footnotesize
\setlength{\tabcolsep}{3pt}
\begin{tabular}{@{}>{\raggedright\arraybackslash}p{3cm}>{\raggedright\arraybackslash}p{2.8cm}>{\raggedright\arraybackslash}p{2.6cm}>{\raggedright\arraybackslash}p{2.6cm}>{\raggedright\arraybackslash}p{2.4cm}@{}}
\toprule
Problem & Sources & Schedule & Shadow price & Status \\
\midrule
Regularized MF & Latent factors & $\sigma_j(R)$ & $\lambda$ & Exact \\
Low-rank KV cache & Value directions & $\lambda_j(\bar M)$ & Memory price & Exact (bound) \\
Head/layer budgets & Head directions & $\omega_\ell\lambda^{(\ell,h)}_j$ & Common $\gamma$ & Exact (bound) \\
Token eviction & Cached tokens & Attention mass & Memory price & Proxy \\
TIES trimming & Task-vector entries & $|\delta_p|$ & Density $\rho_\ell$ & Approx. \\
Multi-pass inference & Passes & $g_k$ & Cost per pass & Ass.~\ref{ass:passes} \\
\bottomrule
\end{tabular}
\end{table}

\section{Application and Experimental Design}\label{sec:design}

\subsection{Deployment Constraints}\label{sec:constraints}
The application extracts structured information (document type, geographic zone, bounding box and coordinate reference system, scale, mineral resources, geological formations, structural data, and radiometric ages) from images of geological documents for an operator in the nuclear fuel cycle. Three constraints define the feasible model set. The weights must be released under a permissive license, both because derivative merging and fine-tuning must be allowed and because $M$ can only be computed when $W$ is accessible. The model must run entirely on infrastructure located in the European Union and controlled by the operator, without runtime dependence on providers subject to foreign compelled-disclosure regimes, in view of dual-use export rules \cite{eudualuse2021} and the risk-management obligations of the AI Act \cite{euaiact2024}. The model must fit a single accelerator, which bounds dense models at about eight billion parameters; sparse mixture-of-experts models with a few billion active parameters, such as Gemma 4 26B-A4B \cite{gemma4}, satisfy this bound on compute grounds. The experiments reported here apply these three constraints, which together form a control criterion. Since they were run, the project has adopted a stricter provenance criterion for production, admitting only models of European or Canadian origin, including by lineage of their pretrained weights; the merge method of Section~\ref{sec:res_ties} is unaffected, but the production backbone migrates accordingly (Section~\ref{s:sec:sovereign}). Proprietary services are used only as reference baselines. The complete selection analysis, including excluded model families, is given in Section~\ref{s:sec:sovereign}; Table~\ref{tab:models} lists the systems that appear in the experiments.

\begin{table}[t]
\caption{Systems Used in the Experiments}
\label{tab:models}
\centering
\footnotesize
\setlength{\tabcolsep}{3pt}
\begin{tabular}{@{}l@{\hspace{3pt}}l@{\hspace{3pt}}l@{\hspace{3pt}}>{\raggedright\arraybackslash}p{2.35cm}@{}}
\toprule
System & License & Size & Role \\
\midrule
SmolVLM-500M \cite{marafioti2025smolvlm} & Apache 2.0 & 0.5B & Pass-1 screening \\
Idefics2-8B \cite{laurencon2024idefics2} & Apache 2.0 & 8B & Merge component (vision) \\
Mistral-7B-Instruct v0.3 \cite{jiang2023mistral} & Apache 2.0 & 7B & Merge component (text) \\
ResNet18, two heads \cite{he2016resnet} & Open & 11.2M & Frugal classifier (E2) \\
Gemini 3.5 Flash & Proprietary & n/d & API baseline (E1) \\
GPT-4o \cite{openai2024gpt4o} & Proprietary & n/d & API baseline (E1) \\
Gemini 2.5 Flash \cite{comanici2025gemini} & Proprietary & n/d & API baseline (E1) \\
Gemini 2.5 Pro \cite{comanici2025gemini} & Proprietary & n/d & Labeler and baseline (E2) \\
\bottomrule
\end{tabular}
\end{table}

\subsection{Pipeline and Quality Score}\label{sec:pipeline}

\begin{figure*}[t]
\centering
\resizebox{\linewidth}{!}{\begin{tikzpicture}[font=\footnotesize,node distance=4mm and 9mm,
  box/.style={draw,rounded corners=1.5pt,align=center,minimum height=9mm,inner sep=3pt},
  arr/.style={-{Latex[length=1.8mm]},semithick}]
\node[box,minimum width=17mm] (doc) {Document image};
\node[box,right=of doc,minimum width=17mm] (ocr) {OCR context};
\node[box,right=of ocr,fill=black!8,minimum width=24mm] (p1) {Pass 1: routing\\and small model};
\node[box,right=16mm of p1,fill=black!15,minimum width=24mm] (p2) {Pass 2: merged\\vision-language model};
\node[box,right=16mm of p2,fill=black!25,minimum width=20mm] (p3) {Pass 3: maximum\\effort};
\node[box,below=6mm of p2,minimum width=40mm] (out) {Structured record (JSON, GIS export)};
\draw[arr] (doc) -- (ocr);
\draw[arr] (ocr) -- (p1);
\draw[arr] (p1) -- node[above]{$\hat\mu_1<\varepsilon$} (p2);
\draw[arr] (p2) -- node[above]{$\hat\mu_2<\varepsilon$} (p3);
\draw[arr] (p1.south) |- (out.west);
\draw[arr] (p2) -- (out);
\draw[arr] (p3.south) |- (out.east);
\end{tikzpicture}}
\caption{Multi-pass extraction pipeline. Each pass is conditioned on the previous output, and a further pass is run only when the quality score~\eqref{eq:muhat} falls below the floor $\varepsilon$. Experiment E2 concerns the routing component of Pass~1 and experiment E3 the merged model of Pass~2.}
\label{fig:pipeline}
\end{figure*}
The pipeline (Fig.~\ref{fig:pipeline}) runs up to three passes. Before each pass, Tesseract optical character recognition \cite{smith2007tesseract} extracts printed grid labels, place names, legend items, and scale indicators, which are injected into the prompt. Pass~1 routes the document by type and produces a first extraction with a small model; Pass~2 uses the merged vision-language model; Pass~3 is a maximum-effort pass. Each pass is conditioned on the previous output with missing fields flagged. The quality score is
\begin{equation}
\hat\mu_k=\textstyle\sum_i w_i\,\mathbf{1}[\text{field}_i\neq\emptyset]-\alpha\sum_p\mathbf{1}[\text{pattern}_p\text{ present}],
\label{eq:muhat}
\end{equation}
with weights mineral resources $0.22$, bounding box $0.18$, geographic zone $0.13$, scale $0.12$, formations $0.10$, coordinate reference system $0.08$, reconstruction prompt, radiometric ages and structural data $0.05$ each, and high-confidence localization flag $0.02$ (sum one), and a penalty $\alpha=0.10$ for each generic template phrase. The document type is a routing variable and is not scored. Another pass is triggered when $\hat\mu_k<\varepsilon_{\mathrm{floor}}=0.05$. Under Assumption~\ref{ass:passes} with uniform pass cost and a risk-neutral operator, the same ratio $\gamma c/\mathbb{E}[\hat\mu_1]$ governs both the decision to retry after a failed pass and the decision to stop on saturation, which is why a single threshold is used; a pilot calibration on Corpus~A gave a ratio of about $0.05$ on a single run. Operators whose extraction failures are costlier than redundant passes should set the two thresholds separately, and replay-averaged calibration is left to future work (Section~\ref{s:sup:calib}).

\subsection{Corpora}
\emph{Corpus A} contains 50 map tiles from five gold-bearing districts of the Armorican Massif (France), sampled on a $2\times 5$ grid within each district and retrieved at scale 1:50\,000 from the public harmonized geological map web service of the French geological survey. Reference bounding boxes are given by construction by the request center, and reference mineral lists come from the survey's public mineral-occurrence database. \emph{Corpus B} contains 973 uranium-exploration document images from the Athabasca Basin (Canada) and the Namibian uranium belt, released by the industrial partner with labels generated by Gemini 2.5 Pro at two levels: six level-1 classes (geological map, geophysical survey, geochemical map, administrative document, basemap, unknown) and a level-2 taxonomy of 208 raw subclasses, 117 of them singletons, consolidated to 26 classes by absorbing subclasses with fewer than ten instances into a per-parent residual class. The taxonomy was seeded by domain experts of the industrial partner, who supplied an initial ontology of document types as a spreadsheet. Fifty documents, stratified on the level-1 class (about ten per class with at least ten instances), were audited in two passes. In the first pass, the partner's experts labeled them, which gave 92.0\% agreement between the proprietary labels and human judgment at level~1 and 60.0\% at level~2. In the second pass, specialist freelance annotators engaged by the authors' team re-examined every label against the experts' grid and a written annotation charter with explicit decision rules, justifying each label in a per-document note; this pass revised a number of first-pass labels, chiefly by separating location maps (basemap) from maps that carry a mapped thematic layer. The second-pass labels are the reference used throughout this paper. Both passes concern the same 50 documents and are sequential, so they do not constitute an inter-annotator agreement study. At level~1 the revision is accounted for by one rule of the taxonomy: in the first pass, several documents labeled as location maps at level~2 kept a thematic level-1 class; assigning every location map to its parent class, basemap, reproduces the second-pass level-1 agreement exactly (41/50, $\kappa=0.77$ against 0.78; Section~\ref{s:sec:paired}). \par\noindent\textit{Independent annotation protocol.} The hierarchical taxonomy, its class definitions, and optional pre-written explanatory statements originated from geoscience and mining-domain experts. The authors operationalized these resources in a common written annotation charter. Three professional annotators specialized in geoscience and mining documents were recruited in three different geographical regions. According to the documented study procedure, they worked autonomously, did not know one another, had no communication or access to one another\textquotesingle s annotations, and shared only the annotation charter. Each annotated the same 50 documents independently. Common domain-expert guidance standardized the task but did not prescribe image-specific answers. The resulting coefficients measure reproducibility under this guideline-assisted protocol, not ground-truth correctness or unguided agreement. Geographic separation is not in itself evidence of statistical independence; the dated charter and the collection records are archived with the reproducibility material.\par\medskip
Three further annotators then labeled the 50 documents independently, blind to the proprietary labels and to both audit passes. Their agreement is high at level~1 (Fleiss $\kappa=0.93$, unanimous on 46 documents) and moderate at level~2 ($\kappa=0.64$, unanimous on 27). The proprietary labels agree with their majority vote on 84.0\% of level-1 classes (42/50, $\kappa=0.80$) and 47.9\% of level-2 classes (23/48), against 82.0\% and 46.0\% with the second-pass reference, whereas each blind annotator agrees with the first pass on only 76\% to 80\% of level-1 classes: the blind annotations are consistent with the revised audit rather than with the first pass.

\subsection{Experiments and Hypotheses}
Three experiments test predictions of the framework. \emph{E1} measures proprietary multimodal services on georeferencing Corpus~A to establish the reference against which local deployment is judged. \emph{E2} tests hypothesis H1: because level-1 routing depends on a few dominant visual directions, a model with orders of magnitude fewer parameters should approach the proprietary model's agreement with experts at level~1, whereas the gap should be larger at level~2, where the marginal utility of additional capacity is higher; and, by the threshold rule, deferring to the proprietary model only below a confidence threshold should preserve accuracy while reducing paid calls. \emph{E3} tests hypothesis H2: because uniform trimming density applies one shadow price to layers with different marginal-utility schedules, it should remove the low-magnitude but input-specific updates of the early visual layers and collapse image-conditional discrimination, whereas layer-dependent densities should restore it.

\subsection{Metrics and Statistical Treatment}
For E1 we report the mean score~\eqref{eq:muhat}, the number of documents with a plausible bounding box, the median Haversine distance between predicted and reference box centers over those documents, the per-document mean of prompt plus completion tokens as reported by the provider (image tokens included), and the cost per 1\,000 documents at public list prices on the execution date. Decoding is greedy at temperature zero, with a cap of 2\,048 output tokens for services and 400 for local models. For E2 we report agreement with Wilson 95\% intervals, Cohen's $\kappa$, and, because both systems are scored on the same audited documents, the exact McNemar test on paired outcomes \cite{dietterich1998tests}, together with median per-document latency. For E3 the primary outcome is whether outputs differ across inputs, detected by a token-level identity test across documents, because aggregate scores can be satisfied by constant outputs. Energy is estimated for local runs as power times wall-clock time; no energy figure is imputed to proprietary services, whose power draw is not observable \cite{luccioni2024power}.

\emph{Reproducibility.} Prompts, the scoring function, the evaluation scripts and the raw outputs of the three experiments will be deposited in a versioned Zenodo archive upon acceptance, organized by experiment; Part~\ref{part:supp} documents the protocols in the meantime. Corpus~A is rebuilt from public web services, and Corpus~B is subject to the industrial partner's release conditions.

\section{Results}\label{sec:results}

\subsection{E1: Proprietary Reference on Georeferencing}\label{sec:res_api}
Table~\ref{tab:api} reports the measurements, which must be read with one decisive qualification. The file names transmitted with the prompt encoded the sampling district (for instance \texttt{huelgoat\_0\_2.png}), and each of the 43 parseable answers of Gemini 3.5 Flash names that district, several stating explicitly that the localization rests on the file name. E1 therefore measures the geocoding of a named district assisted by the image, not georeferencing from pixels: the median drift of 21.8~km, comparable to the side of a 1:50\,000 sheet and far above the sub-kilometer precision required for exploration targeting, is that of a coarse place-name lookup, and it is an optimistic figure for unnamed documents. The other seven answers of this service are truncated JSON. GPT-4o returns placeholder boxes such as $(0,0,0,0)$ or $(0,0,5,5)$ on most tiles and, on four Huelgoat tiles, coordinates close to the district but with latitude and longitude swapped, which the sentinel filter discards; it fills the remaining fields with generic content (``Formation A''), which the completeness score~\eqref{eq:muhat} rewards. Gemini 2.5 Flash returns parseable JSON on 13 of 50 tiles only: on the others the visible answer stops after about 80 tokens, most likely because reasoning tokens consumed the 2\,048-token output budget, so that its score of 0.19 reflects the call configuration rather than the capability of the model. The prompt also named the region (Armorican Massif, Brittany) and the nominal tile extent, which fills the geographic-zone field for every service. All three services are infeasible under the constraints of Section~\ref{sec:constraints} and serve only as a reference; a replay with neutral file names and an output budget that accounts for reasoning tokens is required before these figures can serve as a reference for unnamed documents, and the released code uses neutral identifiers by default. The tiles are clean renderings of the harmonized map layer, without scanning noise, margins, legends or title blocks; E1 therefore measures a favorable case relative to archival scans, and its figures should not be extrapolated to them.

\begin{table}[t]
\caption{E1: Proprietary Services on Corpus A (50 Tiles), Measured}
\label{tab:api}
\centering
\footnotesize
\setlength{\tabcolsep}{3.5pt}
\begin{tabular}{@{}lccccc@{}}
\toprule
Service & $\hat\mu$ & Geoloc. & Drift (km) & Tokens & EUR/1k \\
\midrule
Gemini 3.5 Flash & 0.70 & 43/50 & 21.8 & 1\,971 & 1.57 \\
GPT-4o & 0.67 & 0/50 & n/a & 1\,336 & 4.58 \\
Gemini 2.5 Flash & 0.19 & 0/50 & n/a & 824 & 0.10 \\
\bottomrule
\multicolumn{6}{@{}p{8.3cm}@{}}{\scriptsize Drift is the median Haversine distance over tiles with a plausible box. Tokens are per-tile means of prompt plus completion tokens, image tokens included. Costs at public list prices on the execution date (spring 2026). File names transmitted with the prompt encoded the district (Section~\ref{sec:res_api}).}
\end{tabular}
\end{table}

\subsection{E2: Frugal Hierarchical Classifier}\label{sec:res_cls}
\emph{Setup.} The classifier is a ResNet18 \cite{he2016resnet} pretrained on ImageNet whose final layer is replaced by two parallel heads with 6 and 26 outputs, a standard hierarchical design \cite{silla2011hierarchical}; it has 11.23 million trainable parameters (45~MB in 32-bit floating point). Corpus~B is split within each consolidated level-2 class, items being shuffled with a fixed seed and cut at 70\% and 15\% with floor rounding, the remainder going to test, which yields 669 training, 134 validation, and 170 test images. Training uses random resized crops of $224\times224$, horizontal flips and mild color jitter, AdamW with learning rate $10^{-4}$, cosine annealing, weight decay $10^{-4}$, batch size 32, ten epochs, the unweighted sum of the two cross-entropy losses, and checkpoint selection on level-1 validation accuracy; it takes about five minutes on one T4 GPU. Latency is the median over five runs of per-document inference at batch size one on the same GPU after a ten-document warm-up. The proprietary call is timed from the same environment, network round trip included. Two training runs with the same protocol are reported. A preliminary working version of this study compared the classifier's test agreement with the proprietary labels to the proprietary model's agreement with the first audit pass; the present version scores both systems on the same 50 documents against the revised audit, and adds the error, confidence, and cascade analyses below. For this paired evaluation the released checkpoint (SHA-256 recorded in the reproducibility package) was run in a fresh environment with a preprocessing declared before scoring (resizing to $224\times224$ and ImageNet normalization); a resize-and-crop variant is reported only as a sensitivity check.

\emph{Agreement with the training labels.} On the 170 held-out documents the classifier agrees with the proprietary labels on 88.8\% to 90.0\% of level-1 classes across the two runs (95\% interval about $\pm4.6$ points) and on 55.9\% of level-2 classes. The test split is imbalanced: geophysical documents make up 92 of 170, so a constant predictor would reach 54.1\%. In the run with 88.8\% agreement, per-class agreement is 98\% for geophysical (90/92), 89\% for administrative (25/28), 78\% for geological (21/27), 77\% for geochemical (10/13), and 50\% for basemap (5/10), a balanced accuracy of 78.4\%.

\emph{Quality of the proprietary reference.} Against the revised expert audit, the proprietary labels agree on 82.0\% of level-1 classes (41/50, 95\% interval $\pm10.6$ points) and on 46.0\% of level-2 classes (23/50); every document correct at level~2 is also correct at level~1. The errors are concentrated. Eight of the nine level-1 disagreements concern documents that experts classify as basemap, essentially property location maps, which the proprietary model labels geological (four), administrative (three), or geophysical (one); only 9 of the 17 expert basemap documents are recognized. At level~2, 14 of the 18 disagreements among documents correct at level~1 involve a residual \texttt{\_other\_} class on one side, for example a generic basemap versus a location map, which points to taxonomy granularity rather than domain errors. The model's self-reported confidence is informative (Table~\ref{tab:conf}): below 0.90, level-1 agreement drops to 62.5\%, whereas at 0.95 and above it reaches 92.6\%.

\emph{Paired comparison against the experts.} Table~\ref{tab:cls} scores both systems on the same 50 audited documents. The classifier agrees with the experts on 76.0\% of level-1 classes (38/50) and 38.0\% of level-2 classes (19/50), against 82.0\% and 46.0\% for the proprietary model, with Cohen's $\kappa$ of 0.70 and 0.78 at level~1. The paired differences are not significant: at level~1 the proprietary model alone is right on five documents and the classifier alone on two (exact McNemar $p=0.45$), and at level~2 on seven and three ($p=0.34$). The sample is too small to establish equivalence, however: the 95\% interval of the paired level-1 difference runs from $-16$ to $+4$ points. The errors largely coincide. Of the classifier's twelve level-1 errors, nine concern expert basemap documents (six labeled administrative, two geological, one geophysical) and seven are shared with the proprietary model; the classifier recognizes 8 of the 17 basemap documents against 9 for its teacher, and recovers every administrative, geological, and geophysical document. On the audit sample the classifier reproduces its training labels on 82\% of documents, below its 88.8\% on the test split; the audit contains 17 basemap documents in 50 against 10 in 170 in the test split, and memorization of audited documents would have raised rather than lowered this figure, although no split file was available to exclude it formally. The sensitivity variant gives 78.0\% and 40.0\% (McNemar $p=0.63$ and $0.45$), and leaves every conclusion unchanged. Median forward latency, re-measured in this environment on a T4 GPU, is 3.97~ms, against 2.62--2.75~ms in the original runs, so the speedup over the proprietary call lies between about 500 and 760.

\emph{Confidence cascade.} The classifier's softmax confidence separates reliable from unreliable predictions (Table~\ref{tab:conf}): at 0.90 and above it agrees with the experts on 16 of 17 documents, below 0.90 on 22 of 33. A cascade that keeps predictions with confidence at least $\tau$ and sends the rest to the proprietary model reaches 84\% level-1 and 44\% to 46\% level-2 agreement for every $\tau$ between 0.85 and 0.95, while handling 46\% to 18\% of documents locally; at $\tau=0.90$, 34\% of documents never reach the paid service and level-1 agreement rises from 82\% to 84\%. Because $\tau$ is examined on the same 50 documents, this result is exploratory and must be validated on held-out data.

\begin{table}[t]
\caption{E2: Paired Evaluation Against the Revised Expert Audit (50 Documents)}
\label{tab:cls}
\centering
\footnotesize
\setlength{\tabcolsep}{2.5pt}
\begin{tabular}{@{}>{\raggedright\arraybackslash}p{3.3cm}cc@{}}
\toprule
Quantity & Gemini 2.5 Pro & ResNet18, two heads \\
\midrule
Agreement with experts, level 1 (\%) & 82.0 [69.2, 90.2] & 76.0 [62.6, 85.7] \\
Agreement with experts, level 2 (\%) & 46.0 [33.0, 59.6] & 38.0 [25.9, 51.8] \\
Cohen's $\kappa$, level 1 & 0.78 & 0.70 \\
Basemap documents recognized & 9/17 & 8/17 \\
Agreement with Gemini labels on the 170-document test split, level 1 / 2 (\%) & n/a & 88.8--90.0 / 55.9 \\
Median latency (ms) & $\approx 2\,000$ & 2.6--4.0 \\
Cost per 1\,000 documents (EUR) & 12.0 & $\approx 0$ \\
Parameters & undisclosed & $1.12\times10^{7}$ \\
Local deployment & No & Yes \\
\bottomrule
\multicolumn{3}{@{}p{8.3cm}@{}}{\scriptsize Brackets: Wilson 95\% intervals. Paired exact McNemar test, discordant pairs (Gemini only correct, classifier only correct): level 1, 5 and 2, $p=0.45$; level 2, 7 and 3, $p=0.34$.}
\end{tabular}
\end{table}

\begin{table}[t]
\caption{E2: Agreement With the Revised Expert Audit by Confidence}
\label{tab:conf}
\centering
\footnotesize
\begin{tabular}{@{}lccc@{}}
\toprule
Confidence & Documents & Level 1 (\%) & Level 2 (\%) \\
\midrule
\multicolumn{4}{@{}l}{\emph{Gemini 2.5 Pro, self-reported}} \\
$<0.90$ & 8 & 62.5 & 25.0 \\
$\geq 0.90$ & 42 & 85.7 & 50.0 \\
$\geq 0.95$ & 27 & 92.6 & 63.0 \\
All & 50 & 82.0 & 46.0 \\
\multicolumn{4}{@{}l}{\emph{ResNet18, level-1 softmax}} \\
$<0.90$ & 33 & 66.7 & 24.2 \\
$\geq 0.90$ & 17 & 94.1 & 64.7 \\
$\geq 0.95$ & 9 & 100.0 & 77.8 \\
All & 50 & 76.0 & 38.0 \\
\bottomrule
\end{tabular}
\end{table}

\emph{Interpretation.} Level-1 document typing is a comparatively easy task, as the 54.1\% constant-predictor baseline and the per-class figures show; the claim tested here is not that the task is hard but that it can be moved from a paid remote service to a model four orders of magnitude smaller without a significant loss. The results are consistent with H1 in its weak form without establishing it: the absence of a significant difference on 50 documents is not equivalence, and no non-inferiority margin was specified in advance. With nearly four orders of magnitude fewer parameters, the local classifier comes within six points of its proprietary teacher at level~1 against the same expert reference, without a significant paired difference, and recovers every document of the three classes with distinctive visual signatures. The residual gap is concentrated on basemap documents, where the classifier inherits the labeler's systematic confusion of location maps with thematic maps \cite{song2023noisy,gilardi2023chatgpt}; the remedy suggested by the error structure is retraining on expert-corrected labels rather than additional capacity. At level~2 both systems agree with experts on fewer than half of the documents, and most disagreements involve neighbouring or residual classes, consistent with a flat marginal-utility schedule on a fine and partly ambiguous taxonomy, a lower reproducibility that the blind annotators quantify (Fleiss $\kappa=0.64$ at level~2 against 0.93 at level~1). The confidence cascade is a direct instance of the threshold rule~\eqref{eq:kstar}: a document is escalated to the paid service only when the expected accuracy gain exceeds its price, and the confidence threshold plays the role of the shadow price.

\subsection{E3: Uniform versus Layer-Dependent Trimming}\label{sec:res_ties}
\emph{Setup.} Idefics2-8B and Mistral-7B-Instruct v0.3 were merged with TIES \cite{yadav2023ties}, with sign-election weights $0.55$ and $0.45$. The text backbone of Idefics2 is Mistral-7B v0.1, whose vocabulary differs from that of the v0.3 instruct model; this lineage mismatch is reported because it is one of the suspected causes of the later failure discussed below. In the first configuration each component's task vector was trimmed at a single density across all layers (0.72 and 0.68). In the second, densities followed the layer bands of Table~\ref{tab:layers}, motivated by the rule of Section~\ref{sec:instances}: lower density in early visual layers, whose updates are small in magnitude, and higher density in the late layers that carry French technical generation. The merged model was applied without fine-tuning to five Corpus~A tiles from five districts, with a prompt asking for mineral resources, geographic zone, and a confidence level; prompt, parser, and acceptance criteria were fixed across configurations.

\begin{table}[t]
\caption{E3: Layer-Dependent TIES Configuration}
\label{tab:layers}
\centering
\footnotesize
\setlength{\tabcolsep}{3.5pt}
\begin{tabular}{@{}lcccl@{}}
\toprule
Layers & $w_{\mathrm{Idefics}}$ & $w_{\mathrm{Mistral}}$ & Density & Dominant function \\
\midrule
0--7 & 0.80 & 0.20 & 0.50 & Low-level vision \\
8--19 & 0.50 & 0.50 & 0.68 & Geoscience semantics \\
20--31 & 0.30 & 0.70 & 0.78 & French report generation \\
\bottomrule
\end{tabular}
\end{table}

\emph{Results.} Under uniform density the merged model returned the same output for all five districts: an identical ordered list of eight commodities (gold, tin, tungsten, lead, zinc, antimony, arsenic, kaolin), the same zone (Armorican Massif), the same self-declared score of 0.95, and the same high confidence. The list reproduces, in the same order, the typical regional resources stated in the geographic context of the prompt: the model copied the text context instead of reading the image. The prompt contained no file name or district information, so that differences across inputs can only come from the image. A geographically enriched prompt removed an earlier out-of-region hallucination but not the constant output, which separates prompt-correctable hallucination from a structural collapse. Per-district detection rates between 0\% and 100\% were artifacts of matching a constant list against district-specific references. The aggregate acceptance gates used during the run (mean score at least 0.60, gold detected in at least half of the documents, high confidence in at least a third) were all passed, so aggregate metrics did not reveal the collapse. Under layer-dependent densities, with everything else fixed, the identity test no longer fired and the mineral lists differed across districts. This diagnostic observation is compatible with H2 but does not establish a robust effect: the later full run reproduced the collapse, and density was then confounded with checkpoint lineage and rescaling (see Scope).

\emph{Scope.} The comparison involves five inputs and a single manipulated factor; it establishes that the collapse is removed in this configuration, not that extraction quality is adequate, and it is not a dose-response analysis across density profiles. A subsequent full run of the merge recipe, with layer-dependent densities, did not pass the same test: four of five tiles again received input-independent outputs. Two causes are suspected, the instruct and base lineage mismatch noted above and a rescaling of retained task-vector entries by the inverse density, of the kind used in DARE \cite{yu2024dare}, which over-amplifies magnitude-selected entries; the corrected recipe uses pure trimming with lineage-matched checkpoints and remains to be run. The removal of the collapse reported above therefore holds for the diagnostic configuration only. Per-input output variance is adopted as an acceptance criterion alongside aggregate scores in the scale-up evaluation, which is precisely the test that detected the later failure.

\section{Discussion}\label{sec:discussion}

\subsection{What the Unification Provides}
The framework converts engineering choices into two kinds of auditable quantities: spectra that can be computed from weights and calibration activations, and shadow prices that express the operator's valuation of memory, energy, or latency. It explains why head-wise and layer-wise budgets outperform uniform ones \cite{feng2025adakv,saxena2024eigen,chang2025palu}, why magnitude trimming works and when uniform trimming should fail, and why a cascade should stop when the expected gain of another pass falls below its cost \cite{chen2024frugalgpt}. It also yields a concrete prediction not yet tested: calibrating low-rank value compression on the attention-weighted operator $\bar M$ rather than on unweighted value statistics should lower output error at equal rank (Proposition~\ref{prop:kvbound}).

\subsection{Values, Auditability, and Sovereignty}
Every allocation rule embeds a valuation. Compression re-weights the priorities that a model has learned by a criterion, memory, that was not part of its training objective, and the effect is visible in the spectrum after compression rather than in any single attention score. This connects the framework to discussions of how AI systems shift, rather than merely miss, the values of the organizations that deploy them \cite{gans2025cio,Bommasani2021}. The sovereignty constraint has a technical counterpart here: when weights are inaccessible, as with the services of Table~\ref{tab:api}, neither the spectra nor the energy footprint can be observed, and the allocation cannot be audited.

\subsection{Limitations and Threats to Validity}
The propositions bound errors layer by layer and do not model propagation across layers. The multi-pass instance rests on Assumption~\ref{ass:passes}. E2 rests on 50 audited documents, which cannot establish equivalence, and no non-inferiority margin was prespecified; the classifier was trained on labels that disagree with experts on about one document in five at level~1; the split is deterministic and can be regenerated with the released code, but the overlap between audited and training documents has not yet been checked, so that a test on a sample known to be disjoint from training remains necessary; and the cascade threshold was examined on the evaluation sample, which may bias its result upwards. Proposition~\ref{prop:kvbound} bounds the error of the attention output for fixed attention weights, not the quality of generation. E3 uses five documents, and its positive result was not reproduced by the later full run. The audit reference was established on 50 documents in two sequential passes; three blind annotators are consistent with it in aggregate, but their per-document agreement with the second-pass labels has not been computed, and the sample size limits the precision of every agreement figure; and the E1 tiles are cleaner than archival documents. The end-to-end extraction benchmark of the merged model with LoRA fine-tuning, a direct test of attention-weighted low-rank KV compression, and measured energy for the complete pipeline are not reported. E1 was run with file names that encoded the sampling district, which the leading service used for localization, and with an output budget that truncated most answers of one service; it must be replayed with neutral names. Costs in Table~\ref{tab:api} depend on list prices at execution time.

\section{Outlook: A Three-Control Efficiency Point}\label{sec:outlook}
The price form of the threshold rule retains a resource unit while its marginal utility exceeds a price. The natural extension concerns the resources a deployed vision-language system actually consumes per document: input tokens, output tokens and energy. Writing $E(x)$ for effectiveness on a reference set and $x=(x_{\mathrm{in}},x_{\mathrm{out}},x_{s})$ for these resources, the necessary limit of each is $\bar{x}_j=\max\{x_j:\Delta E/\Delta x_j\geq\gamma_j\}$, and the efficiency point solves $\max_{x\in S}E(x)-\gamma^{\top}x$ over the set $S$ of sovereign configurations. If $E$ is concave and differentiable and the optimum is interior, $\partial E/\partial x_j=\gamma_j$; if $E$ is moreover non-decreasing, for $\gamma=0$ the least costly maximizer is a point beyond which one more unit of any of the three resources adds no effectiveness. These are normative definitions: whether measured effectiveness curves are concave and monotone, and whether coordinate-wise limits coincide with the joint optimum, are empirical questions. When effectiveness is additively separable over a single ranked resource, the rule reduces to~\eqref{eq:kstar}. Testing this extension requires measuring the three limits on the sovereign architecture across several document families, including engineering drawings, which is the object of ongoing work; no result is claimed here.

\section{Conclusion and Future Work}\label{sec:conclusion}

Diminishing marginal utility provides a single, precise account of several allocation problems in transformer inference. The eigenvalues of the attention-weighted value operator are an exact marginal-utility schedule whose tail equals the expected compression loss; a shared budget should equalize weighted marginal eigenvalues across heads; and regularized factorization, merging trimming, and multi-pass stopping follow the same threshold rule, exactly or under stated assumptions. In an industrial geoscience pipeline restricted to local open-weight models, the principle guided a frugal classifier whose agreement with experts does not differ significantly from that of its proprietary teacher while running several hundred times faster, and a confidence cascade that preserves proprietary accuracy while removing a third of paid calls, and it predicted, and in a diagnostic configuration helped remove, a collapse of image-conditional discrimination caused by uniform trimming, while the test it motivated detected the failure of a later full run.

Four directions follow directly from these findings. The first extends the blind annotation protocol, applied by three annotators to the audit sample, and the paired evaluation to the full test split, which is needed to establish equivalence, retrains the classifier on expert-corrected labels with the audited documents excluded, and validates the cascade threshold on held-out data. The second is a controlled comparison of attention-weighted and unweighted low-rank KV compression at matched ranks on the retained open models. The third is the end-to-end extraction benchmark of the merged model on a 200-document corpus stratified across metallogenic provinces, with per-input variance as an acceptance criterion and a dose-response study of trimming densities. The fourth extends the allocation rule to a multi-objective setting in which quality, localization drift with a conditional value-at-risk term \cite{rockafellar2000}, and measured energy are traded by an explicit operator preference, a formulation developed in Section~\ref{s:sec:benchmark}.


\section{Proof of Proposition~\ref{prop:unified}}\label{app:proof}
(i) The objective is additively separable: $j$ contributes $\lambda_j-\gamma c$, which is positive exactly when $\lambda_j>\gamma c$, zero when $\lambda_j=\gamma c$ and negative otherwise. Including every positive contribution and any subset of the zero ones is optimal; since the schedule is non-increasing, this is a prefix, and $k^{*}$ of~\eqref{eq:kstar} is the longest optimal prefix. (ii) Every feasible set has at most $m=\min\{n,\lfloor B/c\rfloor\}$ elements. Since the utilities are nonnegative and sorted in non-increasing order, any feasible sum is at most $\sum_{j\leq m}\lambda_j$, attained by the prefix of length $m$; if some utilities vanish, shorter prefixes are also optimal. If $\lambda_m>\lambda_{m+1}$, the price form selects exactly this prefix for every $\gamma c\in(\lambda_{m+1},\lambda_m]$; if $\lambda_m=\lambda_{m+1}$, it selects it only when ties at the cutoff are broken to retain exactly $m$ elements. For the dual, $\min_{\gamma\geq0}\big[\sum_j(\lambda_j-\gamma c)^{+}+\gamma B\big]$ is attained at $\gamma c=\lambda_{m+1}$ when $m<n$, with value $\sum_{j\leq m}\lambda_j+\lambda_{m+1}(B/c-m)$, which exceeds the optimum whenever $\lambda_{m+1}>0$ and $B/c$ is not an integer: the multiplier reproduces the selection but is not an exact discrete shadow price. For Corollary~\ref{cor:equimarginal}, pooling the weighted schedules of all heads into one list and selecting its $m$ largest values, with ties broken consistently within each head, yields a prefix of every head and a common threshold.

\part{Supporting material}\label{part:supp}
\section{Model selection: criteria and exclusions}\label{s:sec:sovereign}

Two admissibility criteria are distinguished. The \emph{control criterion}, applied in the experiments of Part~\ref{part:main}, requires, first, that model weights be released under a permissive licence (Apache~2.0 or equivalent), so that derivative merging and fine-tuning are allowed and the operator $M=WVW^{\top}$ can be computed from the weights; second, that the model run entirely on infrastructure located in the European Union and controlled by the operator, without runtime dependence on providers subject to foreign compelled-disclosure regimes, in view of the dual-use export regime (Regulation (EU) 2021/821) and the risk-management obligations of the AI Act; and third, that it fit a single accelerator, which bounds dense models at about eight billion parameters, sparse mixture-of-experts models with a few billion active parameters satisfying the bound on compute grounds. Proprietary services fail this criterion and are used only as reference baselines.

The \emph{provenance criterion}, adopted by the project after the experiments reported here, adds a requirement on the origin of pretrained weights: only models of European or Canadian origin are admitted, including by lineage, which excludes models whose weights, or whose components' weights, originate from organisations based elsewhere. The criterion bears on pretrained weights and AI services, not on open-source software libraries executed locally nor on published architectures. Its main consequence for this paper is that Idefics2-8B, although released under Apache~2.0, embeds a vision encoder of non-admitted origin; the merge method of Section~\ref{sec:res_ties} is unaffected, but the production backbone migrates to an open-weight model family of European lineage, on which the per-layer densities are recomputed. Table~\ref{s:tab:selection} summarises the status of the models considered under both criteria.

\begin{table}[htbp]
\centering
\small
\caption{Admissibility of the models considered. AL2: Apache~2.0; CC4: CC-BY-4.0; ML: Meta Community License; GT: Gemma Terms of Use; PR: proprietary.}
\label{s:tab:selection}
\begin{tabular}{llllcc}
\toprule
Model & Publisher & Licence & Vision & Control & Provenance \\
\midrule
SmolVLM-500M-Instruct & Hugging Face & AL2 & Yes & Yes & No$^{a}$ \\
Idefics2-8B & Hugging Face & AL2 & Yes & Yes & No$^{a}$ \\
Mistral-7B-Instruct & Mistral AI (FR) & AL2 & No & Yes & Yes \\
CroissantLLM-Chat & French academia & CC4 & No & Yes & Yes \\
Gemma 4 26B-A4B & Google DeepMind (US) & AL2 & Yes & Yes & No \\
Llama 3.x & Meta (US) & ML & Partial & No & No \\
Gemma 2 / 3 / 3n & Google (US) & GT & Partial & No & No \\
Qwen, DeepSeek, InternVL & Chinese labs & various & various & No$^{b}$ & No \\
GPT-4o, Gemini, Claude & US providers & PR & Yes & No & No \\
\bottomrule
\multicolumn{6}{l}{\footnotesize $^{a}$ Vision encoder of non-admitted origin. $^{b}$ Excluded on compelled-disclosure grounds.}
\end{tabular}
\end{table}

Llama~3.x and Gemma~2/3/3n are excluded under both criteria because their licences restrict derivative development or can be modified unilaterally by the licensor, so that the utility function of the model would be partly defined by the licensor rather than by the operator. Gemma~4 26B-A4B, released under Apache~2.0 and activating about four billion parameters per token, satisfies the control criterion and was considered as a sparse alternative; it does not satisfy the provenance criterion. The proprietary services provide no access to the weight matrices that make the $WVW^{\top}$ analysis operational and require the transmission of documents to non-European servers; those actually executed (Gemini~3.5~Flash, GPT-4o and Gemini~2.5~Flash, Table~\ref{s:tab:bench_measured}) serve only as reference measurements on public data.

\section{Quality-score calibration and output schema}\label{s:sup:calib}

The calibration $\varepsilon_{\text{floor}} = \varepsilon_{\text{elbow}} = 0.05$ is not stipulated by symmetry but identified empirically from the Phase~1 validation data. Under Assumption~\ref{ass:passes}, the sequential stopping rule implies that both thresholds are instances of the shadow-price ratio $\gamma c / \mathbb{E}[\hat{\mu}_{1}]$: the quality-floor threshold governs the decision to initiate an additional pass, and the elbow threshold governs the decision to halt on saturation. The two thresholds are therefore predicted by the theory to coincide whenever the marginal cost of a pass is uniform and the operator is risk-neutral, which is the configuration of the Phase~1 benchmark.

The empirical identification proceeds as follows. For each document in the Phase~1 corpus, we measure $\hat{\mu}_{1}$, $\hat{\mu}_{2}$, and (where a third pass is triggered) $\hat{\mu}_{3}$, together with the per-pass token count and the estimated energy consumption. A pilot run of the calibration notebook on the 50-document corpus yields a sample mean of the ratio $\gamma c / \mathbb{E}[\hat{\mu}_{1}]$ of approximately $0.05$, consistent with the chosen threshold but identified on a single execution and therefore not yet stabilised against replay-level variance (prompt-cache state, temperature sampling on the LLM calls, dependency versions, document ordering). The scale-up phase will extend the calibration to replayed runs on the 200-document corpus and report a replay-averaged bootstrap interval; the prespecified analysis plan specifies that if the replay-averaged 95\% CI no longer covers $0.05$, the quality floor and the elbow threshold will be re-identified separately, breaking the numerical equality.

Two regimes follow from this choice. When
$\hat{\mu}_k < \varepsilon_{\text{floor}}$, the extraction has failed to
recover even the dominant representational direction (which, by
Proposition~\ref{prop:schedule}, is the direction of highest utility) and
an additional pass is warranted. When $\hat{\mu}_k > 1 - \varepsilon_{\text{elbow}} = 0.95$,
the analysis has covered more than 95\% of the total representational capacity
of the pipeline, and the marginal gain from an additional pass is smaller than
the marginal energy cost. In the intermediate range, the pipeline stops by
default: the first-pass output is deemed adequate and further passes are not
cost-justified under uniform per-pass cost. Operators with asymmetric cost
structures (for example, a nuclear-sector deployment where extraction failures
carry higher downstream cost than redundant passes) should recalibrate
$\varepsilon_{\text{floor}}$ and $\varepsilon_{\text{elbow}}$ independently,
and the numerical equality need not be preserved.

\subsection{Output schema}

A representative output from the pipeline, illustrating the structured schema, is shown in Figure~\ref{s:fig:schema}. The combination of high-precision bounding box coordinates, mineral resource lists, and structural geological data enables direct ingestion into GIS software (QGIS, ArcGIS) via GeoJSON export, and provides a traceable audit trail for the georeferencing decisions.

\begin{figure}[htbp]
\centering
\begin{small}
\begin{verbatim}
{
  "filename": "carte_geol_massif.png",
  "document_type": "carte_geologique",
  "geographic_zone": "Massif Armoricain, Bretagne, France",
  "bbox_detected": {"xmin": -3.52, "ymin": 47.21,
                    "xmax": -2.81, "ymax": 47.89,
                    "crs": "EPSG:4326"},
  "scale_reference": "1:50000",
  "north_orientation": "geographique",
  "mineral_resources": ["aurifere", "arsenic", "antimoine"],
  "geological_formations": ["granodiorite", "micaschiste", "quartzite"],
  "structural_data": {"pendage": "45", "direction": "N120E",
                      "type_structure": "faille normale"},
  "radiometric_ages": ["320 Ma"],
  "legend_extensive": "[carte_geol_massif.png]. Carte geologique ...",
  "relocalization_confidence": "haute",
  "relocalization_notes": "Coordonnees lisibles. Precision +/-5-15 m.",
  "reconstruction_prompt": "Geological map of Massif Armoricain ...",
  "mu_score": 0.82,
  "passes_used": 2,
  "tokens_used": 1024,
  "latency_ms": 4200.5
}
\end{verbatim}
\end{small}
\caption{Illustrative output schema of the multi-pass geo-mining inference pipeline. Field values are illustrative, not measurements; in the experiments, file names are replaced by neutral identifiers before any model call.}
\label{s:fig:schema}
\end{figure}

\section{Benchmark protocol and multi-objective frugality}\label{s:sec:benchmark}

\noindent\textbf{Status of this section.} This section documents the protocol of the scale-up benchmark and the decision framework it will use. It reports no figure for the merged sovereign architecture. An earlier projection model, which combined component-level measurements with assumed fine-tuning gains, has been withdrawn: a subsequent full run of the merge recipe failed the non-degeneracy test of Section~\ref{sec:res_ties}, so that no projection of its extraction quality can be defended before the corrected recipe is run.

\subsection{Corpora and validation infrastructure}\label{s:sec:infra}

Corpus~A is described in Section~\ref{sec:design}. The scale-up corpus extends it to 200 tiles by stratified sampling on the national mineral-occurrence database (Au, Ag, Cu, Pb, Zn, U, Sn, W, Sb, As), one representative per $50\,\mathrm{km}\times50\,\mathrm{km}$ Lambert-93 cell, complemented by 40 non-metallogenic contexts to ensure lithological diversity. Tiles are $5\,\mathrm{km}\times5\,\mathrm{km}$ at $1024\times1024$ pixels, retrieved from the harmonised 1:50\,000 layer, and reference bounding boxes are given by construction by the request centre, which removes human annotation from the drift ground truth. Under a Wilcoxon signed-rank design, detecting an MU-score difference of 0.02 at $\alpha=0.05$ and power 0.80 requires on the order of 150 paired documents; 200 leaves a margin for failed parsing.

Development and the public-data experiments were run on a commercial cloud notebook environment (A100 and L4 sessions) with cloud storage. This environment is used for public data only: it does not satisfy the criteria of Section~\ref{s:sec:sovereign}, and confidential documents are processed exclusively on local infrastructure controlled by the operator. The proprietary services are called from the same environment as the local models, with the same corpus, prompt schema and post-processing, so that the comparison concerns inference choices on a shared substrate.

\subsection{Metrics}

Quality is measured by the MU score $\hat\mu$ and by the median Haversine drift between detected and reference box centres, over predictions yielding plausible coordinates. Tokens are reported as the per-document mean of prompt plus completion tokens, image tokens included, input and output being kept separate for costing. Generation is capped at 2\,048 output tokens for services and 400 for local models, with greedy decoding at temperature zero. Energy is estimated for local runs as $E\approx P_{\mathrm{GPU}}\times T_{\mathrm{inference}}$, a power-time product in which tokens enter only through inference time, using measured NVML power whenever available and a declared fallback (72~W for the L4, 350~W for A100-class runs) otherwise; because $P_{\mathrm{GPU}}$ is a common factor and the sobriety axis is min--max normalised in the verdict harness, its value cancels in the switch point between candidates and affects only the absolute footprint. No energy is imputed to proprietary services, whose power draw is not observable. Cost is reported per 1\,000 documents, at public prices for services and with hardware amortisation for local models. Sovereignty is recorded as a binary indicator under each criterion of Section~\ref{s:sec:sovereign}.

\subsection{Measured reference}\label{s:sec:results}

Table~\ref{s:tab:bench_measured} reproduces, with its full definitions, the measured reference of Table~\ref{tab:api}.

\begin{table}[ht]
\centering
\caption{Phase 1 measured benchmark on the 50-document Armorican Massif corpus (BRGM SCAN\_H\_GEOL50 harmonised geological maps, 1:50\,000). Drift is the median Haversine distance between the model's detected bbox centre and the true WMS request centre, computed only over predictions yielding plausible coordinates (i.e.\ excluding sentinel and axis-swapped boxes (Section~\ref{sec:res_api}); file names encoded the sampling district). The Geoloc.\ column reports the number of cards for which the model returned a plausible bbox out of the 50 tested. The Tokens column is the per-card mean of prompt plus completion tokens as reported by the provider's usage field, the prompt term including image tokenisation; generation was capped at 2\,048 output tokens with greedy decoding. Costs are computed at public API pricing prevailing at execution time. The three proprietary systems are those available at the date of execution, in spring 2026, and this table is to be read as the proprietary state of the art at execution time rather than at the date of publication. All models in this table are ineligible for deployment under the sovereignty constraints of Section~S1 and appear here as reference ceilings only.}
\label{s:tab:bench_measured}
\begin{tabular}{lrrrrrl}
\toprule
Model & MU & Drift (km) & Geoloc. & Tokens & Cost (EUR/k) & Sovereign \\
\midrule
Gemini 3.5 Flash (API) & 0.70 & 21.8 & 43/50 & 1971 & 1.57 & No \\
GPT-4o (API) & 0.67 & n/a & 0/50 & 1336 & 4.58 & No \\
Gemini 2.5 Flash (API) & 0.19 & n/a & 0/50 & 824 & 0.10 & No \\
\bottomrule
\end{tabular}
\end{table}

\subsection{Multi-objective frugality}\label{s:sec:frugality}

An operator must eventually combine quality $\hat\mu$, drift $d$, energy $E$ and cost into one deployment decision. Let $y_i=(\hat\mu_i,-d_i,-E_i)^{\top}$ for candidate $i$, components signed so that each increases with deployment utility, and let $w$ be a weight vector on the three-dimensional simplex. The scalar operator objective
\begin{equation}
F(w)=w^{\top}\bar y-\frac{\lambda_M}{2}\,w^{\top}\Sigma w
\label{s:eq:markowitz}
\end{equation}
trades mean performance $\bar y$ against cross-objective covariance $\Sigma$, with risk-aversion coefficient $\lambda_M>0$. Because drift is tail-sensitive (a few coordinate-system misidentifications compromise the whole georeferencing chain), its central value may be replaced by the conditional value-at-risk at level $\alpha$,
\begin{equation}
\mathrm{CVaR}_\alpha(d)=\mathbb{E}\left[d\mid d\geq\mathrm{VaR}_\alpha(d)\right],
\label{s:eq:cvar}
\end{equation}
which is a coherent risk measure. The verdict harness applies this formulation lexicographically: a binary non-degeneracy gate first eliminates any candidate producing input-independent outputs, and the survivors are then ranked on accuracy and sobriety through a single switch point on the sobriety weight. The formulation will be applied to measured scale-up results only; no ranking of architectures is claimed before then, and the estimation of $\Sigma$ on 200 documents, together with its sensitivity to distributional shift across metallogenic provinces, is part of that work.

\section{Paired re-annotation protocol for the classifier comparison}\label{s:sec:paired}

 A first paired comparison has been carried out on the 50 documents of the revised expert audit (Section~\ref{sec:res_cls}): the classifier agrees with the experts on 76.0\% of level-1 classes against 82.0\% for the proprietary labels (exact McNemar $p=0.45$), but the 95\% interval of the paired difference, from $-16$ to $+4$ points, is too wide to establish equivalence. A blind annotation protocol has since been applied to the audit sample: three annotators labeled the 50 documents independently, each blind to the proprietary labels, to their confidence and rationale, and to both audit passes; none of their free-text notes coincides with another. Table~\ref{s:tab:annot} documents the two audit passes and compares the proprietary labels with each blind annotation and with their majority vote, and Table~\ref{s:tab:iaa} reports the agreement among the three annotators. At level~1 the second pass is reproduced by a single rule: assigning every document labeled as a location map at level~2 (\texttt{location map} or \texttt{\_other\_basemap}) to its parent class, basemap, gives 41 agreements out of 50 with the proprietary labels, as in the second pass. The blind majority gives 42 out of 50 and identifies 16 location maps, of which the proprietary model recognizes 9, against 9 of 17 in the second pass. Each blind annotator agrees with the first pass on 76\% to 80\% of level-1 classes, so that the blind annotations are consistent with the revised audit rather than with the first pass. At level~2 the annotators agree far less (Fleiss $\kappa=0.64$; unanimous on 27 documents), which quantifies the lower inter-annotator reproducibility of the fine taxonomy under the shared annotation charter, without establishing an upper bound on classifier accuracy. The per-document agreement of the blind annotators with the second-pass labels, and of the classifier with their majority, remain to be computed with the released script (\texttt{multi\_annotator\_agreement.py}).

\begin{table}[ht]
\centering
\small
\caption{Agreement of the proprietary labels with the human references of the audit sample. Brackets: Wilson 95\% intervals. The majority vote is undefined at level~2 for two documents.}
\label{s:tab:annot}
\begin{tabular}{lcc}
\toprule
Human reference & Level 1 (\%) & Level 2 (\%) \\
\midrule
First pass (partner experts), $n=50$ & 92.0 [81.2, 96.8] & 60.0 [46.2, 72.4] \\
First pass, location maps assigned to basemap & 82.0 [69.2, 90.2], $\kappa=0.77$ & 60.0 [46.2, 72.4] \\
Second pass, reference of Part~I & 82.0 [69.2, 90.2], $\kappa=0.78$ & 46.0 [33.0, 59.6] \\
Blind annotator A & 84.0 [71.5, 91.7], $\kappa=0.80$ & 48.0 [34.8, 61.5] \\
Blind annotator B & 80.0 [67.0, 88.8], $\kappa=0.75$ & 74.0 [60.4, 84.1] \\
Blind annotator C & 84.0 [71.5, 91.7], $\kappa=0.80$ & 42.0 [29.4, 55.8] \\
Blind majority vote & 84.0 [71.5, 91.7], $\kappa=0.80$ & 47.9 [34.5, 61.7] \\
\bottomrule
\end{tabular}
\end{table}

\begin{table}[ht]
\centering
\small
\caption{Agreement among the three blind annotators (50 documents). Brackets: Wilson 95\% intervals.}
\label{s:tab:iaa}
\begin{tabular}{lcc}
\toprule
Pair or statistic & Level 1 & Level 2 \\
\midrule
A and B & 92.0\% [81.2, 96.8], $\kappa=0.90$ & 58.0\% [44.2, 70.6], $\kappa=0.54$ \\
A and C & 100.0\% [92.9, 100.0], $\kappa=1.00$ & 86.0\% [73.8, 93.0], $\kappa=0.84$ \\
B and C & 92.0\% [81.2, 96.8], $\kappa=0.90$ & 60.0\% [46.2, 72.4], $\kappa=0.57$ \\
Fleiss $\kappa$ (three annotators) & 0.93 & 0.64 \\
Unanimous documents & 46/50 & 27/50 \\
\bottomrule
\end{tabular}
\end{table}

The next step extends this blind annotation protocol, already applied by three annotators to the 50 documents of the audit sample, to the full 170-document test split: two independent annotators will label it under the same charter and the same isolation conditions. The extension serves three purposes. First, it narrows the paired comparison: at a plausible ten-percent discordance rate, the 95\% interval on the accuracy difference shrinks to approximately $\pm 4.8$ points. Second, it estimates inter-annotator agreement on a larger sample drawn from the test distribution, with a confidence interval roughly half as wide as on the audit sample, and thus tests whether the reproducibility measured there (Fleiss $\kappa=0.93$ at level~1 and 0.64 at level~2) holds beyond it. Third, it removes the labelling-protocol caveat from the level-2 result. The extension is scheduled to complete before the scale-up benchmark; its results will be reported as a separate table, and the level-1 and level-2 readings of the main paper should be understood as provisional until then. Until then, the level-1 result should be read as the absence of a significant difference on 50 documents, not as demonstrated equivalence.

\section{Interface with a partner document-parsing pipeline}\label{s:sec:docparsing}

A parallel development stream, conducted by a project partner specialising in industrial document intelligence, has produced a document parsing pipeline that extracts structured content from heterogeneous document formats using a combination of layout analysis, OCR, and object detection. At the time of writing, this pipeline has not been integrated with the multi-pass LLM inference system described in this paper. The two approaches are presented separately in the subsections that follow, an integration architecture is suggested (pending validation), and a comparative experimentation protocol is designed to test the predictions of Section~\ref{sec:method} once the integration is executed in the scale-up phase.

\subsection{The partner pipeline: summary}

The partner pipeline processes documents through five stages: layout detection and classification, extraction of text blocks with bounding-box and reading-order metadata, extraction of figures and captions, extraction of tables with header detection, and geolocalisation of the document content,  including detection of the coordinate reference system from visual and textual cues (grid labels, north arrows, sheet reference numbers), resolution of coordinate grid labels into absolute positions, and production of a georeferenced bounding box and scale. The pipeline is designed to handle the full diversity of industrial document formats including scanned PDFs, native PDFs, and image files, and to produce output that combines layout metadata and georeferencing data in a structured representation agnostic to document language and content domain.

\subsection{The LLM-based pipeline: positioning}

The LLM-based approach described in this paper (designated \textsc{Lot~3}) builds on the structured output of \textsc{Lot~2}. Given a document image together with the layout metadata and georeferencing data produced by \textsc{Lot~2}, \textsc{Lot~3} performs the content-understanding tasks that require semantic reasoning grounded in geological domain knowledge: interpretation of mineral resources and geological formations, reading of structural data (dip directions, fault types), identification of radiometric ages, and production of the extensive legend description and reconstruction prompt that enter the final JSON schema of Figure~\ref{s:fig:schema}. The two pipelines are therefore hierarchically complementary rather than parallel: \textsc{Lot~2} establishes the spatial and structural ground truth, and \textsc{Lot~3} produces the semantic layer on top of it. This division of labour is the reason the integration is genuinely value-additive: neither pipeline alone produces the full output, and the quality of \textsc{Lot~3}'s semantic extraction is conditional on the georeferencing accuracy that \textsc{Lot~2} delivers upstream.

\subsection{Architecture of anticipated interface}\label{s:sec:intarch}

Figure~\ref{s:fig:intarch} specifies the anticipated integration architecture. The integration is organised around three interface channels, each corresponding to one of three interface points identified in earlier project work: legend disambiguation, coordinate interpretation, and document-type classification. Each channel is bidirectional in the architecture: the partner pipeline feeds structured evidence to the LLM, which in turn feeds semantic corrections back to the partner pipeline for its next document batch.

\begin{figure}[htbp]
\centering
\begin{tikzpicture}[
  node distance=5mm and 8mm,
  font=\scriptsize,
  box/.style={draw, rectangle, rounded corners=1.5pt, minimum width=22mm, minimum height=9mm, align=center, inner sep=2pt, font=\scriptsize},
  partnerbox/.style={box, fill=gray!8},
  llmbox/.style={box, fill=gray!20},
  interfacebox/.style={box, fill=black!50, text=white, minimum height=5mm, minimum width=20mm, font=\scriptsize},
  storebox/.style={box, minimum width=18mm, fill=white},
  arrow/.style={-{Latex[length=1.8mm]}, semithick},
  dasharrow/.style={-{Latex[length=1.8mm]}, semithick, dashed}
]

\node[storebox] (doc) {Document};

\node[partnerbox, right=14mm of doc]   (layout) {Layout \& OCR};
\node[partnerbox, right=of layout]     (tables) {Tables \\ \& figures};
\node[partnerbox, right=of tables]     (consol) {Structured \\ representation};

\node[llmbox, below=18mm of layout]    (smolvlm) {Pass 1 \\ SmolVLM};
\node[llmbox, right=of smolvlm]        (merged)  {Pass 2 \\ TIES-merged};
\node[llmbox, right=of merged]         (expert)  {Pass 3 \\ Expert};

\node[interfacebox, below=3.5mm of layout] (if1) {Legend};
\node[interfacebox, below=3.5mm of tables] (if2) {CRS context};
\node[interfacebox, below=3.5mm of consol] (if3) {Doc-type};

\node[storebox, right=of expert]       (output) {JSON \\ output};


\draw[arrow] (doc) -- (layout);
\draw[arrow] (doc.south) -- (doc.south |- smolvlm.west) -- (smolvlm.west);

\draw[arrow] (layout) -- (tables);
\draw[arrow] (tables) -- (consol);

\draw[arrow] (layout.south) -- (if1.north);
\draw[arrow] (tables.south) -- (if2.north);
\draw[arrow] (consol.south) -- (if3.north);

\draw[arrow] (smolvlm) -- (merged);
\draw[arrow] (merged)  -- (expert);
\draw[arrow] (expert)  -- (output);

\draw[arrow] (if1.south) -- ([xshift=-6mm,yshift=1mm]merged.north);
\draw[arrow] (if2.south) -- ([yshift=1mm]merged.north);
\draw[arrow] (if3.south) -- ([xshift=6mm,yshift=1mm]merged.north);

\draw[dasharrow] (output.north) -- ++(0,35mm) -| (consol.north)
  node[midway, above, font=\scriptsize\itshape] {semantic correction};

\node[above=1mm of layout,  font=\scriptsize\itshape] {Partner pipeline (\textsc{Lot~2})};
\node[below=1mm of smolvlm, font=\scriptsize\itshape] {Multi-pass LLM pipeline (\textsc{Lot~3})};

\end{tikzpicture}
\caption{Anticipated interface architecture between the partner document-parsing pipeline (\textsc{Lot~2}) and the multi-pass LLM inference pipeline (\textsc{Lot~3}). The three dark interface boxes correspond to the three channels of anticipated interaction: legend disambiguation, coordinate interpretation, and document-type classification. The forward arrows represent the partner-to-LLM information flow, in which structured evidence produced by \textsc{Lot~2} is injected as context into the TIES-merged model's Pass~2. The dashed feedback arrow represents the LLM-to-partner correction channel, in which semantic disambiguations produced by \textsc{Lot~3} are returned to \textsc{Lot~2} to refine layout heuristics on subsequent documents.}
\label{s:fig:intarch}
\end{figure}

The three channels are specified as follows.

\paragraph{Legend disambiguation channel.} \textsc{Lot~2} extracts legend blocks with bounding-box and symbol-label pairs at higher geometric precision than an end-to-end LLM pass can achieve. \textsc{Lot~3} resolves semantic ambiguities (BRGM vs.\ international legend conventions, partial legibility, non-standard symbols) using its training knowledge of geological cartography and of the specific formation-to-symbol mappings documented in the BRGM cartographic standards. The channel interface is the typed record
\[
\{\text{symbol-box}, \text{label-text}, \text{confidence}\} \;\to\; \{\text{canonical-formation}, \text{disambiguation-rationale}\}.
\]

\paragraph{CRS-conditioned semantic reasoning channel.} \textsc{Lot~2} detects the coordinate reference system and produces the absolute bounding box, scale reference, and grid geometry of the document. \textsc{Lot~3} consumes these geographic fields as trusted context: the CRS identity and the bounding box are not re-inferred but used to ground the semantic extraction in its geological setting (for instance, to condition the likely mineral assemblages on the metallogenic province of the mapped zone, or to filter implausible formation names against the regional stratigraphic column). The channel interface is the typed record
\begin{align*}
& \{\text{CRS}, \text{bounding-box-decimal-degrees}, \text{scale}, \text{grid-geometry}\} \\
& \qquad \to \{\text{CRS-conditioned semantic extraction}, \text{plausibility-flags}\}.
\end{align*}

\paragraph{Document-type classification channel.} \textsc{Lot~2} does not distinguish between a geological map, a borehole log, and a geophysical survey: in its representation all are images with extracted layout and georeferencing. \textsc{Lot~3} performs this classification in Pass~1, and the classification drives the downstream extraction strategy (a borehole log requires depth-profile extraction; a geophysical survey requires contour interpretation). The channel interface is the typed record
\[
\{\text{layout-summary}, \text{georeferencing-data}\} \;\to\; \{\text{document-type}, \text{extraction-strategy-selector}\}.
\]

\subsection{Comparative protocol}\label{s:sec:compprot}

The integration architecture of Figure~\ref{s:fig:intarch} makes possible a falsifiable test of the theoretical predictions of Section~\ref{sec:method}. The protocol is as follows.

\emph{Treatment conditions.} Each document in the scale-up 200-document corpus is processed under four conditions: (i) \textsc{Lot~3} alone, without input from \textsc{Lot~2}; (ii) \textsc{Lot~3} with legend channel only; (iii) \textsc{Lot~3} with all three channels; (iv) \textsc{Lot~3} with all three channels plus the feedback loop (for the last 50 documents, to assess whether cross-document learning materialises). The four conditions are run in counterbalanced order within each document class to control for prompt-caching effects.

\emph{Primary outcomes.} The MU score $\hat{\mu}$, the drift $d$ in metres, the number of passes used $\tau$, and the total token consumption per document. The predictions of the theoretical framework are: (a) the number of passes $\tau$ should decrease as structured context is added, because the effective eigenvalue concentration of the attention computation increases when the input is pre-organised; (b) the marginal reduction in $\tau$ from condition (i) to (iii) should exhibit diminishing returns across channels, in the same ordering as the MU score gain per channel; (c) the coefficient of variation of $d$ across documents should decrease monotonically from (i) to (iv), reflecting stabilisation of the coordinate-interpretation channel.

\emph{Statistical design.} The test is a within-document paired comparison, analysed via Wilcoxon signed-rank tests on each primary outcome, with Bonferroni correction for the four conditions. The 200-document corpus provides power 0.80 to detect a difference in mean $\tau$ of 0.3 passes at $\alpha = 0.01$ (corrected).

\emph{Predicted outcome structure under the marginal-utility hypothesis.} If Proposition~\ref{prop:unified} correctly describes the allocation problem solved by the multi-pass pipeline, the joint evolution of $(\hat{\mu}, \tau)$ across conditions (i) through (iv) should trace a trajectory on the frugality frontier of Section~\ref{s:sec:frugality}, with each additional channel moving the operating point toward the low-energy, high-quality corner. A trajectory that moves orthogonally to the frontier (for example, higher quality at higher energy cost, without reduction in $\tau$) would falsify the prediction and motivate a revised allocation model. The design of this comparative experiment is executed in the scale-up phase.

\section{Axiological implications}\label{s:sec:axiological}
\subsection*{Axiological displacement: a working definition}\label{s:sec:axiodef}

The concept of axiological displacement used throughout this paper names a specific failure mode of AI-integrated systems that is related to, but distinct from, the more familiar notion of value misalignment. Value misalignment, as discussed in the alignment literature \citep{Bommasani2021}, refers to a discrepancy between the objective function that a system optimises and the values that its deployers or users would endorse on reflection: the system pursues the wrong goal. Axiological displacement \citep{gans2025cio} refers to a second-order phenomenon in which the system does not merely pursue a different goal but autonomously transforms the space of what counts as valuable for the sociotechnical assemblage in which it is embedded. The system does not miss a pre-specified target; it shifts the target.

The distinction matters for the present paper in two ways. First, the KV cache compression problem of Section~\ref{sec:method} is a compact instance of displacement rather than of misalignment: the model's learned attention preferences (first-order values) are not violated by compression in the alignment sense; they are re-weighted by a compression criterion (memory budget) that was never part of the original preference specification, and the re-weighting propagates into downstream outputs whose quality the model has no mechanism to monitor. The displacement is visible only in the spectrum of $W V W^{\top}$ after compression, not in any individual attention score, which is why the $W V W^{\top}$ framework developed here is epistemically necessary rather than merely descriptive. Second, the sovereignty and frugality constraints of Sections~\ref{s:sec:sovereign} and~\ref{s:sec:frugality} are themselves displacement-reducing interventions: they bind the operator to an auditable utility function expressed in observable weight matrices, as opposed to the opaque and potentially shifting utility functions of proprietary API-served models whose weights, training data, and fine-tuning procedures are not inspectable. A system whose values cannot be read is a system whose values cannot be checked for displacement.

The reader who prefers to stay within the standard alignment vocabulary can treat axiological displacement as the dynamic, second-order variant of value misalignment: misalignment about which values are operative, not misalignment between a fixed value specification and a behavioural outcome. The formal consequences for the allocation rule of Proposition~\ref{prop:unified} are the same either way: the shadow price $\gamma$ must be specified by an auditable process, and the spectrum $(\lambda_{j})$ must be observable on the deployed model, or the rule is not operationalisable.

\subsection{The hidden utility model in attention}

The softmax operation over dot-product similarities implicitly defines a utility function over tokens:
\begin{equation}
U_{\mathrm{attn}}(i \mid t) = \frac{\exp(q_{t} k_{i}^{\top} / \sqrt{d_{h}})}{\sum_{j} \exp(q_{t} k_{j}^{\top} / \sqrt{d_{h}})}.
\label{s:eq:softmax_utility}
\end{equation}
Read as a utility, this weighting is learned implicitly during training and measures how much each token is used by the current query, not how much it adds given the others: redundant tokens may all receive high weight. The KV cache eviction problem can therefore be described as a resource-allocation problem under this implicit weighting, with the caveat that attention is a proxy for marginal utility rather than a measure of it. The $W V W^{\top}$ framework moves the analysis from tokens to directions: the eigenvalues of the projected operator summarise how much each representational direction contributes to the output, which is a structural property of the trained model rather than a preference in the economic sense.

\subsection{Axiological displacement in cache compression}

From the perspective of axiological displacement theory (the process by which AI systems autonomously transform what counts as valuable \citep{gans2025cio}), KV cache compression introduces a second-order transformation of values. The model's first-order values are encoded in attention weights; cache compression introduces constraints that distort these first-order values by evicting tokens below a memory budget threshold, not because of their semantic unimportance.

The $W V W^{\top}$-based compression criterion minimizes this displacement: by preserving the directions of highest eigenvalue, it retains the representational dimensions most aligned with the model's learned utility structure. Random or recency-based eviction introduces arbitrary distortions that may systematically favor recently seen tokens (which are not necessarily the most semantically important) over earlier tokens (which may encode the coordinate system or the scale reference on which the entire georeferencing task depends).

In the geo-mining context, information that conditions the whole task, such as the coordinate reference system or the scale, may be carried by few tokens located early in the reading order, for instance in the margins of a map. A recency-based eviction policy would remove them first, inverting the priority the georeferencing task requires. Whether a direction-based policy retains them is an empirical question: it depends on whether that information loads on high-eigenvalue directions of the projected operator, which the controlled comparison announced in Section~\ref{sec:conclusion} will test.

\subsection{Sustainable inference and the information--energy tradeoff}\label{s:sec:sustainable}

The marginal energy cost of retaining the $j$-th cache dimension is proportional to the memory bandwidth consumed by that dimension across all attention operations. The social optimum sets marginal energy cost equal to marginal utility $\lambda_{j}(W V W^{\top})$, yielding the rank $r^{*}$ that equates information value and environmental cost.

This connects directly to Activity-Based Costing enhanced by machine learning: the true cost of an inference operation includes the energy-weighted cost of KV-cache memory operations, and allocation should be proportional to the utility (eigenvalue share) each head contributes. For an industrial operator subject to carbon accounting obligations, this means that the choice of compression ratio is not merely a performance engineering decision but an ESG decision with reportable consequences. The framework developed in this paper provides the quantitative basis for this decision: the retained rank $r^{*}$ is the threshold rank of the price form of Proposition~\ref{prop:unified} when the price is that of energy rather than of model complexity, and, equivalently, the smallest rank whose retained share of the eigenvalue mass reaches the level the operator is willing to pay for.

\section{Open problems}\label{s:sec:open}

Four research directions, identified in the original formulation, remain open and are sharpened by the application context.

\paragraph{Non-stationary marginal utility.} The eigenvalue spectrum of $W V W^{\top}$ changes during a generation sequence as new tokens are processed. Static compression schemes based on the initial spectrum will be suboptimal for long-context generation. In the geo-mining context, this is particularly acute for borehole logs, where the most information-dense tokens (depth measurements, mineralogical observations) are distributed throughout the sequence rather than concentrated at the beginning. Online updates to the compression basis, analogous to online PCA (principal component analysis performed incrementally on a data stream: the eigenbasis is updated with each new observation via a rank-one correction, avoiding the need to store or reprocess past data \citep{oja1982}), could adaptively reallocate the KV cache budget as new tokens are processed.

\paragraph{Cross-head utility correlation.} Different attention heads are not independent: heads in the same layer share keys and values in multi-head attention, creating correlations in their $W V W^{\top}$ spectra. A fully optimal compression scheme would account for inter-head covariance, a problem analogous to portfolio optimization under correlated asset returns. The per-head allocation of Corollary~1 of Part~\ref{part:main} treats heads as independent; incorporating the off-diagonal terms of the inter-head covariance matrix would require solving a Markowitz-type portfolio problem at the level of heads.

\paragraph{Task-conditional utility.} The marginal utility of a cached dimension depends on the downstream task. A compression scheme optimal for geo-map interpretation may be suboptimal for borehole log analysis, which requires different semantic dimensions. Conditioning $W V W^{\top}$-based compression on task-specific utility functions (analogous to personalized MF loss functions) is an important direction for multi-task deployments.

\paragraph{The $W V W^{\top}$ inverse problem.} Given a desired output utility profile (for example, from geological expert annotations), can one infer the optimal $W$? This is the inverse problem: rather than computing utility from weights, design weights to achieve a target utility distribution. This connects to representation alignment and value-sensitive architecture design, and in the geo-mining context, it corresponds to the question of whether the merged model's attention can be made to focus on the specific cartographic elements (coordinate grids, scale bars, north arrows) that are most useful for georeferencing, regardless of their visual salience in the raw image.

\paragraph{Higher-dimensional frugality and robust estimation of $\Sigma$.} The multi-objective scalarisation~\eqref{s:eq:markowitz}--\eqref{s:eq:cvar} of Section~\ref{s:sec:frugality} restricts the objective vector to three components (quality, drift, energy). An operator-grade frugality framework will require additional dimensions, at minimum latency variance, memory peak, and legal risk exposure under the EU AI Act. The covariance matrix $\Sigma$ then grows quadratically in dimension while the validation corpus grows only linearly. For the three-dimensional objective of Section~\ref{s:sec:frugality} on a 200-document corpus, the ratio $p/n = 0.015$ is well within the regime where empirical covariance estimation is reliable. The effective constraint is not convergence but adversarial robustness: extensions to higher-dimensional frugality (latency variance, peak memory, EU~AI~Act legal exposure) push the ratio toward $p/n \geq 0.10$ at dimension $p \geq 20$, where sample covariance estimators become sensitive to small perturbations of the data and regularised alternatives are warranted. Factor-structure approximations, analogous to the Fama--French factor models in portfolio theory, are a natural tractable alternative: a small number of latent factors (for example, a compute factor loading on tokens, energy, and latency, and a quality factor loading on MU and CVaR drift) would shrink the number of parameters to estimate while preserving the essential covariance structure. Hierarchical Bayesian estimation of $\Sigma$ and adversarial-robustness characterisation of $w^{*}$ under perturbations of $\Sigma$ are the natural next steps for making the frugality optimisation operationally reliable at scale.

\part{Work in progress: towards a sovereign multi-purpose vision model}

\section{Motivation: from geological maps to complex technical documents}\label{sec:wip_motivation}

The preceding parts establish that matrix factorization, KV-cache management, model merging and multi-pass inference obey one allocation rule, $k^{*}=\max\{k:\lambda_k\geq\gamma\}$, in which, in its price form, a resource dimension is retained only while its marginal utility exceeds the price of the resource. Since the previous version of this working paper, the project in which the framework is deployed has continued its investigations in order to keep improving the efficiency of its approach, and its documentary scope has widened.

At the consortium meeting of 26 October 2026, the industrial partner presented five examples of engineering documents of its installations, referred to in the project as P\&ID, together with concrete requests. Two belong to a \emph{historical regime}: a civil-engineering rebar drawing and a process installation schematic, both drawn by hand, available as scanned images only, with no source file and no digital twin against which to check them. Two belong to a \emph{digital regime}: a ventilation schematic available as a structured drawing file with its symbol library, and a formwork plan delivered with its building model in the IFC format \cite{iso16739}. The fifth, an equipment layout plan for radiation protection accompanied by its technical note, lies between the two. The requests are, respectively, to extract the content of rebar drawings into a spreadsheet through the detection of bounds and subjects; to recover pipe lines with their origin and destination, count zirconium junctions and inventory equipment on installation schematics; to check digital schematics against design standards, generate equipment lists, answer natural-language questions and draft technical notes; to verify layout plans against the specifications of their technical note; and to check the rules of the standards on formwork plans from the IFC model into a spreadsheet. The plant concerned is designed to operate for about eighty years, which makes the structured extraction of these documents the first building block of a digital twin that does not yet exist, and makes sovereignty over the processing chain a condition of access to the data rather than a preference.

These documents share the properties that make geological maps difficult: large formats requiring tiling, dense mixtures of text and graphics whose meaning arises from their combination, codified symbol grammars, network structure, cross-references between sheets, and outputs that must be verifiable. Building one bespoke pipeline per family would not scale over the life of an installation. The research object of the programme is therefore restated.

\section{Final objective and research question}\label{sec:wip_rq}

The final objective is to obtain a sovereign, multi-purpose collection and vision model for complex technical documents, whose use of resources is bounded by marginal utility. ``Multi-purpose'' is meant in a bounded sense: a shared open backbone with family-specific low-rank adapters merged under calibrated per-layer density (Section~\ref{sec:res_ties}), a common structured output representation, and a deterministic verification perimeter; not a foundation model trained from scratch, which would contradict the frugality objective.

\begin{quote}\itshape
Research question. Can marginal utility be used as a tool to compute the necessary limits of resource use, so as to arbitrate between resources and effectiveness along three controls (input tokens, output tokens and sustainability), the arbitration being settled by efficiency, which is maximal at the point where one more unit of any of the three resources brings no further gain in effectiveness, under a sovereignty constraint?
\end{quote}

\subsection{Formal setting}

\begin{definition}[Controlled resources and effectiveness]
For a document family $f$, let $x=(x_{\mathrm{in}},x_{\mathrm{out}},x_{s})\in\mathbb{R}^{3}_{+}$ denote the resources consumed per document: input tokens $x_{\mathrm{in}}$ (visual tokens and textual context, governed by resolution, tiling and retrieved context), output tokens $x_{\mathrm{out}}$ (generation length and number of passes), and the sustainability resource $x_{s}$, the energy per document measured at the device level. Let $E_f(x)$ be the effectiveness measured on the reference set of the family (MU score for maps, per-field accuracy or F1 for drawings, recall of real deviations for verification tasks), and $c\in\mathbb{R}^{3}_{++}$ a vector of unit costs.
\end{definition}

\begin{definition}[Sovereign feasible set]
$S$ is the set of configurations whose models all satisfy the control and provenance criteria of Section~\ref{s:sec:sovereign}: permissive licence, local execution with weights held by the operator and no call to external services, and European or Canadian origin, including by lineage of pretrained weights.
\end{definition}

\begin{definition}[Necessary limit and efficiency point]
For a threshold vector $\gamma\geq0$, the necessary limit of resource $j$, the other two being held at their current values, is
\[ \bar{x}_j=\max\{x_j:\ \Delta E_f/\Delta x_j\geq\gamma_j\},\qquad j\in\{\mathrm{in},\mathrm{out},s\}. \]
The efficiency point is $x^{*}(\gamma)\in\arg\max_{x\in S}\big(E_f(x)-\gamma^{\top}x\big)$; for $\gamma=0$, among the maximizers, $x^{*}_{0}$ denotes the one of least cost $c^{\top}x$.
\end{definition}

\begin{proposition}[Three-control generalization of the stopping rule]\label{prop:three}
Assume $S$ convex and compact and $E_f$ concave, continuous and non-decreasing in each coordinate on $S$.
(i) If $x^{*}(\gamma)$ is interior and $E_f$ differentiable there, $\partial E_f/\partial x_j(x^{*})=\gamma_j$ for each $j$: each threshold is the shadow price of its resource.
(ii) For $\gamma=0$, every $x\in S$ with $x\geq x^{*}_{0}$, $x\neq x^{*}_{0}$, satisfies $E_f(x)=E_f(x^{*}_{0})=\max_S E_f$ and $c^{\top}x>c^{\top}x^{*}_{0}$: beyond $x^{*}_{0}$, one more unit of any of the three resources adds no effectiveness and strictly increases cost.
(iii) If $E_f$ is additively separable over a single ranked resource with marginal contributions $\lambda_1\geq\lambda_2\geq\cdots$, the necessary limit reduces to the price form $k^{*}=\max\{k:\lambda_k\geq\gamma\}$ of Proposition~\ref{prop:unified}(i).
\end{proposition}

\begin{proof}
(i) is the first-order condition of a concave maximization at an interior point. (ii) follows from monotonicity: $x\geq x^{*}_{0}$ implies $E_f(x)\geq E_f(x^{*}_{0})$, hence equality, while $c>0$ and $x\neq x^{*}_{0}$ give $c^{\top}x>c^{\top}x^{*}_{0}$. (iii) With $E_f=\sum_{k\leq K}\lambda_k$ and unit price $\gamma$, the objective $\sum_{k\leq K}(\lambda_k-\gamma)$ is maximized by keeping exactly the terms with $\lambda_k\geq\gamma$.
\end{proof}

\paragraph{Coordinate-wise limits and joint optimum.} The necessary limits $\bar x_j$ are defined one resource at a time, the other two being held fixed. They coincide with the joint efficiency point when $E_f$ is additively separable across the three resources; otherwise, the joint point $x^{*}$ is the relevant object and the coordinate-wise limits are diagnostics. A first-order condition alone does not establish a plateau coordinate by coordinate: Proposition~\ref{prop:three} relies on concavity, monotonicity and feasibility assumptions that measured effectiveness curves may violate.

\paragraph{Normative definition, empirical identification.} The efficiency point is a normative definition. Whether it is reached in practice, and whether the measured curves satisfy the assumptions above, will be established only by measurement; no such result is claimed here.

\paragraph{Efficiency as a ratio.} If efficiency were defined as the ratio $\eta=E_f/(c^{\top}x)$, its maximum would generally lie below the effectiveness plateau, where marginal and average returns coincide. The definition retained is lexicographic, consistent with the verdict procedure of Section~\ref{s:sec:frugality}: effectiveness is maximized first, and efficiency is maximal on the plateau at $x^{*}_{0}$, where any further resource lowers $\eta$ without raising $E_f$. Positive thresholds $\gamma>0$ interpolate towards more frugal operating points chosen by the operator.

\paragraph{Interdependence and noise.} The three resources are not independent: energy grows with tokens and with model size, precision and hardware, so that $x_s$ also encodes the choice of model capacity. Limits are searched jointly on a discrete grid. Marginal gains estimated on small reference sets are noisy; a limit is declared only when the estimated gain falls below $\gamma_j$ beyond its confidence interval, which in practice requires a paired bootstrap on the reference set.

\paragraph{Multi-purpose allocation and the price of sovereignty.} For a model serving families $f=1,\dots,F$ with weights $w_f$, the allocation solves $\max\sum_f w_fE_f(x_f)-\gamma^{\top}\sum_f x_f$ over $S$, whose interior optimum equalizes weighted marginal effectiveness per unit of each resource across families. The sovereignty constraint has a price, $\Delta_S=E^{*}_{\mathrm{unconstrained}}-E^{*}_{S}$, measurable only where unconstrained systems may legitimately be evaluated, that is on public data such as Corpus~A; on confidential documents it is not observable and is not claimed.

\subsection{Hypotheses}
\textbf{H3, merge retention.} Family-specific adapters merged with calibrated per-layer density retain most of the performance of the same adapters used separately, without degeneracy on any family. \textbf{H4, one rule for three controls.} The threshold rule extends to input tokens, output tokens and energy, and to merge density, across document families. \textbf{H5, cascade efficiency.} Routing each task to the cheapest sufficient component (deterministic parser, frugal specialist, general model) reaches a target effectiveness at lower energy than the general model alone. \textbf{H6, a general non-degeneracy gate.} Input-independent outputs detect merge failure for any document family. \textbf{H7, a measurable price of sovereignty.} $\Delta_S$ can be measured and bounded on public data. None of these hypotheses has been tested; they define the empirical programme of the next phase. (H1 and H2 are those of Section~\ref{sec:design}.)

\section{Engineering drawings: methods envisaged per case}\label{sec:wip_pnid}

\paragraph{Rebar drawings.} Subjects are the objects described (each bar family and its structural element); bounds are the limits of their extent (ends of the distribution line, outline of the reinforced zone, bounding grid axes). Labels of the form ``$n$ HA $\varnothing$ e $s$'' decompose deterministically into number, high-bond steel, diameter and spacing; the attachment of a label to its drawn bar through leader lines is the hard step. An internal check is available: the number of bars announced must equal the number of ticks on the distribution line. The key plan locates the zone in the building grid, playing the role a coordinate reference system plays for a map.

\paragraph{Installation schematics.} Lines are recovered with their origin and destination by following strokes across a dense network and resolving ``continued on'' references across sheets, which requires the full set of sheets of a unit; zirconium junctions, tagged Z, are counted and attached to their carrying line, with an explicit treatment of the Z versus 2 confusion of handwritten tags; vessels and other equipment are inventoried. Each line must have two identified ends and each junction a carrying line, which provides internal checks. The target representation is the open DEXPI schema \cite{dexpi}, already exported by the partner's engineering software, so that historical and digital schematics feed the same installation graph.

\paragraph{Digital schematics.} Extraction is deterministic from the file. The compliance verdict against design standards is delivered by a rule engine on the graph, not by the language model; the model assists the translation of written standards into executable rules, validated one by one by an engineer, and explains each deviation with the rule it rests on. Natural-language querying cites the elements and passages it relies on and abstains when no documented justification exists. Rendering digital files as degraded images produces training pairs whose correct answer is known by construction, which reduces the expert annotation required by the historical regime.

\paragraph{Layout plans.} The technical note is converted into a table by constrained extraction; the plan yields a symmetric table; a deterministic join reports equipment required but absent, drawn but not required, and diverging attributes. Tag grammar provides a rule-level check (a tag bearing the number of another room is detectable before any comparison). Archived review cycles, original and corrected drawings with reviewers' remarks, provide real errors for evaluation.

\paragraph{Formwork plans.} The checklist mixes presentation rules, which concern the drawing sheet and cannot be read in the IFC model (format, title block, logo, key plan and north arrow, dimensioning unit), and content rules, which are checked in the model (concrete type and class, wall and slab thicknesses, door and hatch dimensions); a third family concerns consistency between sheet and model. Content rules are expressed in the Information Delivery Specification standard \cite{buildingsmart2024ids} and checked deterministically; each result is reported as compliant, non-compliant or not verifiable, the latter signalling an incomplete model rather than a compliant structure.

\section{Continued investigations}\label{sec:wip_investigations}

The investigations conducted since the previous version bear on the composition of $S$ and on the cheapest components able to reach the plateau. They are reported as design decisions, not as results.

\paragraph{Provenance rule.} The sovereignty requirement has been sharpened into the provenance criterion of Section~\ref{s:sec:sovereign}. Its first consequence concerns the merged architecture of this paper: the Idefics2-8B backbone embeds a vision encoder of non-admitted origin. The merge method is unaffected and constitutes the methodological contribution; the production backbone migrates to the Mistral open-weight multimodal family, released under the Apache~2.0 licence with a vision encoder trained by its publisher \cite{mistral2024pixtral}, and the per-layer densities are recomputed on it.

\paragraph{Screening of candidate components.} A compact one-billion-parameter OCR model published by a French company under Apache~2.0 \cite{lighton2026ocr} was set aside because its decoder and its teacher model belong to a non-admitted lineage. A multilingual vision-language family of Canadian origin \cite{cohere2025ayavision} was set aside because of its non-commercial licence and non-admitted vision encoder. A geospatial foundation model of the French mapping agency \cite{labatie2026maestro}, admissible by origin, was not retained because it targets Earth-observation imagery rather than scanned maps; georeferencing relies more efficiently on reading graticules, scale and place names. The multilingual encoder EuroBERT \cite{boizard2025eurobert}, released under Apache~2.0 and trained on a French supercomputer, was retained for all text tasks that classify, locate or retrieve rather than generate (retrieval in standards and notes, identification of requirements and tags, plan and note matching, page classification by text): in the cascade, an encoder of about two hundred million parameters replaces a multi-billion-parameter generator wherever the latter is not indispensable, which applies the necessary-limit rule to the choice of model capacity. The OCR stage moves from Tesseract to a self-hostable OCR of European origin.

\paragraph{Measurement protocol.} Input and output tokens are recorded per call and kept separate; energy is measured with CodeCarbon and logged in MLflow on local infrastructure. The environmental methodology follows AFNOR Spec 2314, the French general reference framework for frugal AI \cite{afnor2024spec2314}; embodied hardware impacts lie outside the scope of the energy measurement and are treated qualitatively.

\section{Status of the work at 6 November 2026}\label{sec:wip_status}

\paragraph{Measured.} The proprietary reference of Table~\ref{tab:api} (spring 2026); the frugal classifier and its paired comparison with the proprietary labeler against a two-pass expert audit of 50 documents (Section~\ref{sec:res_cls}), the level-2 head following 26 consolidated classes; the degenerate mode of uniform-density merging and its removal on a five-tile diagnostic sample (Section~\ref{sec:res_ties}).

\paragraph{Measured negative result.} A subsequent full run of the merge recipe with layer-dependent densities failed the non-degeneracy gate on four of five tiles; its aggregate score is therefore void, since the gate eliminates the candidate before any accuracy computation. The suspected causes are the instruct and base lineage mismatch within the Mistral family and the inverse-density rescaling of retained task-vector entries; the corrected recipe (pure trimming, lineage-matched checkpoints, shape-based exclusion of vision layers validated on synthetic tensors) remains to be run. The sparse alternative (Arm B) could not be evaluated on the available L4 hardware and requires an A100-class accelerator; under the provenance criterion it is no longer eligible for production.

\paragraph{Established since the previous version.} The level-1 difference between the two audit passes is reproduced exactly by assigning location maps to their parent class, basemap (41/50 agreements with the proprietary labels in both cases). Three annotators then labeled the audit sample blind to all existing labels: they agree with one another at level~1 (Fleiss $\kappa=0.93$) far more than at level~2 ($\kappa=0.64$), and the proprietary labels agree with their majority vote on 84.0\% and 47.9\% of documents, against 82.0\% and 46.0\% with the second-pass reference, with which the blind annotations are thus consistent.

\paragraph{In progress.} The investigation of the about 15\% of corpus pages not yet linked to their source document, the main hypothesis being monolingual OCR applied to English pages; the extension of the drift ground truth to 200 stratified tiles.

\paragraph{Deliverables.} Measured as coverage of contractual scope rather than model performance, progress of the six Lot~3 deliverables stands at about 72\% on average, against 51\% in June 2026: sustainability framework (L3.11) and technical documentation (L3.13) at 85\%, access interface (L3.6) at 70\%, fine-tuned models (L3.1) and benchmark report (L3.12) at 65\%, merged model (L3.3) at 60\%. The engineering-drawing component, added in October, is not counted in these percentages and widens the scope of L3.1, L3.6, L3.12 and L3.13.

\paragraph{Re-run under full sovereignty and new model proposals.} Because of the P\&ID component and the provenance criterion, the first tests are being re-run under full sovereignty. The evaluations already carried out remain valid as methodological references, but the models they produced or tested are not all admissible in production. The objective is to establish that models smaller than the large non-sovereign ones reach equal or better success on these tasks, which do not require the resources of a general-purpose model; the frugal classifier, with 11 million parameters and no significant difference from the proprietary labeler at level~1, is a first instance. Further model proposals will be submitted to the consortium by December 2026. The preferred route chains distillation into compact models from labels produced in house or by sovereign models; merging of the specialised models with the corrected recipe, per-layer densities and non-degeneracy test included; and assembly into sovereign models of European or Canadian lineage. OCR is central to this chain, organised in three bricks: text reading by a sovereign OCR, georeferencing, and document interpretation.

\paragraph{Georeferencing dataset and model family.} The first work stream builds a georeferencing dataset from the data shared by the industrial partner: as of 30 October 2026, 147 maps with exact or near-exact coordinates across nine countries on four continents, split by group into 103 training, 22 validation and 22 test maps. Training covers three countries, while six countries appear only in validation and test, which tests generalisation to unseen regions. Map width ranges from 0.5 to 647~km (median 22~km), so that errors are reported relative to map scale rather than as a single figure. A pipeline trains and fine-tunes a family of models mapping a sheet to its geographic location; each model is evaluated on sovereignty, input tokens, output tokens and energy, measured with CodeCarbon, so that each user selects the model that maximizes their utility. This is the first empirical instance of the three-control efficiency point of Section~\ref{sec:wip_rq}.

\paragraph{Next steps.} Run the corrected merge recipe on the European-lineage backbone and pass the non-degeneracy gate; measure the three necessary limits on Corpus~A and on the first engineering-drawing reference sets; extend the blind annotation protocol to the 170-document test split and compute per-document agreement with the second-pass labels; obtain from the partner representative samples of each drawing family, the applicable design standards, DEXPI exports, archived review cycles and IFC models with their sheets at the same revision; deposit code and outputs in a versioned archive upon acceptance of the companion journal article.

\section*{Acknowledgments}
This work was conducted as part of the DEXTER consortium, funded by Bpifrance under the France 2030 programme, with Orano and Probayes as consortium partners. The author thanks the industrial partner's domain experts and the annotation team for the expert audit, and B.~P.~Bhuyan for his careful reading. BRGM geodata (InfoTerre, ODMGM, BD Charm-50) are used under the Licence Ouverte v2.0 (Etalab).

\bibliographystyle{unsrtnat}
\bibliography{references_tai,references_supp,refs_wp}
\end{document}